\documentclass[a4paper]{article}
\usepackage[latin1]{inputenc} % ou \usepackage[utf8]{inputenc}
\usepackage[T1]{fontenc} % ou \usepackage[OT1]{fontenc}
\usepackage{RR,RRthemes}
\usepackage{hyperref}
\usepackage{amsmath}
\usepackage{amssymb}
\usepackage{subfigure}
\usepackage{pgf}
\usepackage{tikz}
\usepackage{algorithm}
\usepackage{algorithmic}
\usetikzlibrary{arrows,shapes,snakes,automata,backgrounds,petri}

\RRdate{August 2010}
\RRauthor{% les auteurs
Julien Mairal\thanks[contrib]{Equal contribution.}%
\thanks[lab]{
INRIA - WILLOW Project, Laboratoire d'Informatique de l'Ecole Normale Sup\'erieure (INRIA/ENS/CNRS UMR~8548).
             23, avenue d'Italie, 75214 Paris. France}
  \and
  Rodolphe Jenatton\thanksref{contrib}\thanksref{lab}
\and 
Guillaume Obozinski\thanksref{lab}
\and
Francis Bach\thanksref{lab}
}

\def\x{{\mathbf x}}
\def\z{{\mathbf z}}
\def\1{{\mathbf 1}}

\def\X{{\mathbf X}}

\def\gammab{{\boldsymbol\gamma}}

\def\varepsilonb{{\boldsymbol\varepsilon}}

\def\y{{\mathbf y}}
\def\w{{\mathbf w}}

\def\W{{\mathbf W}}

\def\N{{\mathcal N}}

\def\GG{{\mathcal G}}

\def\H{{\mathcal H}}

\def\e{{\mathbf e}}

\def\u{{\mathbf u}}

\def\Real{{\mathbb R}}

\def\u{{\mathbf u}}

\def\argmin{\operatornamewithlimits{arg\,min}}

\def\st{~~\text{s.t.}~~}
\def\defin{\triangleq}

\newcommand{\BlackBox}{\rule{1.5ex}{1.5ex}}
\newenvironment{proof}{\par\noindent{\bf Proof.\ }}{\hfill\BlackBox\\[2mm]}

\newcommand{\INPUT}{ \STATE {\textbf{Inputs:}} }

\newcommand{\R}[1]{\mathbb{R}^{#1}}
\newcommand{\RR}[2]{\mathbb{R}^{#1 \times #2}}
\newcommand{\G}{\mathcal{G}}

\newcommand{\NormDeux}[1]{\left\|#1\right\|_2}
\newcommand{\NormInf}[1]{\left\|#1\right\|_{\infty}}

\def \EcartTabFig {\hspace*{-0.25cm}}

\newcommand{\IntSet}[1]{[ 1;#1 ]}%{\textlbrackdbl #1 \textrbrackdbl}%{\text{\textlbrackdbl} #1 \text{\textrbrackdbl}}
\def \xib{{\boldsymbol\xi}}
\def \xibbar{{\boldsymbol{\bar \xi}}}

\def \kappab {{\boldsymbol\kappa}}

\long\def\symbolfootnote[#1]#2{\begingroup\def\thefootnote{\fnsymbol{footnote}}\footnote[#1]{#2}\endgroup} 
\newtheorem{lemma}{Lemma}
\newtheorem{proposition}{Proposition}

\authorhead{Mairal, Jenatton, Obozinski \& Bach}
\RRtitle{Algorithmes de Flots pour Parcimonie Structur\'ee}
\RRetitle{Network Flow Algorithms for Structured Sparsity}
\titlehead{Network Flow Algorithms for Structured Sparsity}
\RRresume{Nous consid\'erons une classe de probl\`emes d'apprentissage r\'egularis\'es par une norme induisant de la parcimonie structur\'ee, d\'efinie comme une somme de normes $\ell_\infty$ sur des groupes de variables. Alors que de nombreux efforts ont \'et\'es mis pour d\'evelopper des algorithmes d'optimisation rapides lorsque les groupes sont disjoints ou structur\'es hi\'erarchiquement, nous nous int\'eressons au cas g\'en\'eral de groupes avec recouvrement. Nous montrons que le probl\`eme d'optimisation correspondant est li\'e \`a l'optimisation de flots sur un r\'eseau. Plus pr\'ecis\'ement, l'op\'erateur proximal associ\'e \`a la norme que nous consid\'erons est dual \`a la minimisation d'un co\^ut quadratique de flot sur un graphe particulier. Nous proposons une proc\'edure efficace qui calcule cette solution en un temps polynomial. Notre algorithme peut traiter de larges probl\`emes, comportant des millions de variables, et ouvre de nouveaux champs d'applications pour les mod\`eles parcimonieux structur\'es. Nous pr\'esentons diverses exp\'eriences sur des donn\'ees d'images et de vid\'eos, qui d\'emontrent l'utilit\'e et l'efficacit\'e de notre approche pour r\'esoudre de nombreux probl\`emes.
}
\RRabstract{ 
We consider a class of learning problems that involve a structured
sparsity-inducing norm defined as the sum of $\ell_\infty$-norms over groups of
variables. Whereas a lot of effort has been put in developing fast
optimization methods when the groups are disjoint or embedded in a specific hierarchical structure, 
we address here the case of general overlapping groups.
To this end, we show that the corresponding optimization problem
is related to network flow optimization. More precisely,
the proximal problem associated with the norm we consider is dual to a
\emph{quadratic min-cost flow problem}.   
We propose an efficient procedure which computes its solution exactly in polynomial time.  
Our algorithm scales up to millions of variables, and opens up a whole
new range of applications for structured sparse models.  We present several
experiments on image and video data, demonstrating the applicability and
scalability of our approach for various problems.

}
\RRmotcle{optimisation de flots, optimisation convexe, m\'ethodes parcimonieuses, algorithmes proximaux}
\RRkeyword{network flow optimization, convex optimization, sparse methods, proximal algorithms}
\RRprojet{Willow}
\RRdomaine{4} % cas du domaine numero 1
\RRthemeProj{willow} % theme du projet Apics
\RRdomaineProjBis{willow} % domaine du projet Apics
\RCParis % Paris Rocquencourt
\begin{document}
\RRNo{7372} 
\makeRR   % cas d'un rapport de recherche
%% \makeRT % cas d'un rapport technique.
%% a partir d'ici, chacun fait comme il le souhaite

\section{Introduction}
Sparse linear models have become a popular framework for dealing with various
unsupervised and supervised tasks in machine learning and signal processing.
In such models, linear combinations of small sets of variables are selected to
describe the data. Regularization by the $\ell_1$-norm has emerged as a
powerful tool for addressing this combinatorial variable selection problem,
relying on both a well-developed theory (see \cite{tsybakov} and references
therein) and efficient algorithms \cite{efron,nesterov,beck}.

The $\ell_1$-norm primarily encourages sparse solutions, regardless of the
potential structural relationships (e.g., spatial, temporal or hierarchical)
existing between the variables.  Much effort has recently been devoted to
designing sparsity-inducing regularizations capable of encoding higher-order
information about allowed patterns of non-zero
coefficients~\cite{jenatton,jacob,zhao,huang,baraniuk}, with successful
applications in bioinformatics~\cite{jacob,kim3}, topic
modeling~\cite{jenatton3} and computer vision~\cite{huang}.
 
By considering sums of norms of appropriate subsets, or \textit{groups}, of
variables, these regularizations control the sparsity patterns of the
solutions.  The underlying
optimization problem is usually difficult, in part because it involves nonsmooth
components.
Proximal methods have proven to be effective in this context, essentially
because of their fast convergence rates and their ability to deal with large
problems~\cite{nesterov,beck}.  While the settings where the
penalized groups of variables do not overlap~\cite{roth2} or are embedded in a
tree-shaped hierarchy~\cite{jenatton3} have already been studied,
sparsity-inducing regularizations of general overlapping groups have, to the
best of our knowledge, never been considered within the proximal method framework.

This paper makes the following contributions: 
\begin{itemize}
\item It shows that the \emph{proximal operator} associated with the structured norm we consider can be computed by solving a \emph{quadratic min-cost flow} problem, thereby establishing a connection
with the network flow optimization literature.
\item It presents a fast and scalable procedure for solving a large
class of structured sparse regularized problems, which, to the best of our knowledge, have not been addressed efficiently before.
\item It shows that the dual norm of the sparsity-inducing norm we consider can also be
evaluated efficiently, which enables us to compute duality gaps for the corresponding optimization problems.
\item It demonstrates that our method is relevant for various
applications, from video background subtraction
to estimation of hierarchical structures for dictionary learning of natural image~patches.
\end{itemize}

\section{Structured Sparse Models}
We consider in this paper convex optimization problems of the form
\begin{equation}
   \min_{\w \in \Real^p} f(\w) + \lambda \Omega(\w), \label{eq:formulation}
\end{equation}
where $f: \Real^p \to \Real$ is a convex differentiable function and
$\Omega: \Real^p \to \Real$ is a convex, nonsmooth, sparsity-inducing regularization function.
When one knows \emph{a priori} that the solutions of this learning problem
only have a few non-zero coefficients, $\Omega$ is often chosen to be the
$\ell_1$-norm, leading for instance to the Lasso~\cite{tibshirani}.
When these coefficients are organized in groups, a penalty encoding
explicitly this prior knowledge can improve the prediction performance
and/or interpretability of the learned models~\cite{roth2,yuan,huang2,obozinski}. Such a penalty might for example take the~form
\begin{equation}
   \Omega(\w) \, \defin\, \sum_{g \in \GG} \eta_g \max_{j\in g}|\w_j| \,=\, \sum_{g \in \GG} \eta_g \|\w_g\|_\infty, \label{eq:def_omega}
\end{equation}
where $\GG$ is a set of groups of indices, $\w_j$ denotes the $j$-th coordinate
of $\w$ for $j$ in $\IntSet{p}\defin\{1,\ldots,p\}$, the vector~$\w_g$ in $\R{|g|}$ represents 
the coefficients of $\w$ indexed by $g$ in $\GG$,
and the scalars~$\eta_g$ are positive weights. A sum of $\ell_2$-norms is also
used in the literature~\cite{zhao}, but the $\ell_\infty$-norm is piecewise
linear, a property that we take advantage of in this paper. Note that when
$\GG$ is the set of singletons of $\IntSet{p}$, we get back the $\ell_1$-norm.

If $\GG$ is a more general \emph{partition} of $\IntSet{p}$, variables are selected in
groups rather than individually.  When the groups overlap, $\Omega$ is still a
norm and sets groups of variables to zero together~\cite{jenatton}.  The latter
setting has first been considered for hierarchies~\cite{zhao,kim3,hkl}, and
then extended to general group structures~\cite{jenatton}.\footnote{Note that
other types of structured sparse models have also been introduced, either
through a different norm~\cite{jacob}, or through non-convex
criteria~\cite{huang,baraniuk}. }  Solving Eq.~(\ref{eq:formulation}) in this
context becomes challenging and is the topic of this paper.
Following~\cite{jenatton3} who tackled the case of hierarchical groups, we
propose to approach this problem with proximal methods, which we now introduce.
 
\subsection{Proximal Methods}
In a nutshell, proximal methods can be seen as a natural extension of
gradient-based techniques, and they are well suited to minimizing the sum $f+\lambda \Omega$ of two convex
terms, a smooth function~$f$ ---continuously differentiable with Lipschitz-continuous gradient--- and a potentially non-smooth function~$\lambda \Omega$ (see \cite{combette} and references therein).
At each iteration, the function $f$ is linearized at the
current estimate~$\w_0$ and the so-called
\textit{proximal} problem has to be solved:
\begin{displaymath}
    \min_{\w \in \R{p}} f(\w_0) + (\w - \w_0)^\top \nabla f(\w_0) + \lambda \Omega(\w) + {\displaystyle \frac{L}{2}}\|\w - \w_0\|_2^2.
\end{displaymath}
The quadratic term keeps the solution in a neighborhood where the current linear
approximation holds, and $L\!>\!0$ is an upper bound on the Lipschitz
constant of $\nabla f$. This problem can be rewritten~as
\begin{equation}\label{eq:prox_problem}
   \min_{\w \in \R{p}} {\displaystyle \frac{1}{2}} \NormDeux{\u-\w}^2 + \lambda'  \Omega(\w),
\end{equation}
with $\lambda' \defin \lambda / L$, and $\u \defin \w_0 - \frac{1}{L} \nabla f(\w_0)$.
We call \textit{proximal operator} associated with the regularization
$\lambda'\Omega$ the function that maps a vector~$\u$ in~$\R{p}$ onto the
(unique, by strong convexity) solution~$\w^{\star}$ of Eq.~(\ref{eq:prox_problem}). Simple proximal method use $\w^{\star}$ as the next iterate, but accelerated variants~\cite{nesterov,beck} are also based on the proximal operator
and require to solve problem (\ref{eq:prox_problem}) \textit{exactly} and \textit{efficiently} to enjoy their fast convergence rates.
Note that when $\Omega$ is the $\ell_1$-norm, the solution of Eq.~(\ref{eq:prox_problem}) is obtained by a soft-thresholding~\cite{combette}.

The approach we develop in the rest of this paper extends \cite{jenatton3} to the case of general
overlapping groups when $\Omega$ is a weighted sum of
$\ell_\infty$-norms, broadening the application of these regularizations to a wider spectrum of
problems.\footnote{For hierarchies, the approach of \cite{jenatton3}
applies also to the case of where $\Omega$ is a weighted sum of
$\ell_2$-norms.}

\section{A Quadratic Min-Cost Flow Formulation}
In this section, we show that a convex dual of problem~(\ref{eq:prox_problem}) for general overlapping groups~$\G$ can be reformulated as a \emph{quadratic min-cost flow problem}. We propose an efficient algorithm to solve
it \textit{exactly}, as well as a related algorithm to compute the dual norm of~$\Omega$.
We start by considering the dual formulation to problem~(\ref{eq:prox_problem}) introduced in~\cite{jenatton3}, for the case where $\Omega$ is a sum of 
$\ell_\infty$-norms:
\begin{lemma}
   [Dual of the proximal problem~\cite{jenatton3}]
\label{lem:dual}~\newline
   Given $\u$ in $\R{p}$, consider the problem
\begin{equation}
   \min_{ \xib \in \RR{p}{|\G|} } \frac{1}{2} \| \u - \sum_{g \in \G}
   \xib^g  \|^2_2 ~~\mbox{ s.t. }~~\forall g\in\G,\
   \|\xib^g\|_1 \leq \lambda \eta_g ~~~\mbox{ and }~~~ \, \xib^g_j = 0 \, \mbox{ if }
   \, j \notin g ,\label{eq:dual_problem}
\end{equation}
where $\xib\! =\! (\xib^g)_{g \in \G}$ is in $\RR{p}{|\G|}$, and $\xib^g_j$ denotes the $j$-th coordinate
of the vector $\xib^g$. Then, every solution $\xib^\star\! =\! (\xib^{\star g})_{g \in \G}$ of Eq.~(\ref{eq:dual_problem}) satisfies $\w^\star\! =\! \u\!-\!\sum_{g \in \G}\xib^{\star g}$, where $\w^\star$ is the solution of Eq.~(\ref{eq:prox_problem}).
\end{lemma}
Without loss of generality,\footnote{
Let $\xib^\star$ denote a solution of Eq.~(\ref{eq:dual_problem}).
Optimality conditions of Eq.~(\ref{eq:dual_problem}) derived in \cite{jenatton3}
show that for all $j$ in $\IntSet{p}$, the signs of the
non-zero coefficients $\xib_j^{\star g}$ for $g$ in $\GG$ are the same as the signs of the entries $\u_j$.  
To solve Eq.~(\ref{eq:dual_problem}), one can therefore flip
the signs of the negative variables $\u_j$, then solve the modified dual formulation
(with non-negative variables), which gives the magnitude of the entries
$\xib_j^{\star g}$ (the signs of these being known).}
we assume from now on that the scalars $\u_j$
are all non-negative, and we constrain the entries of~$\xib$ to be non-negative.
We now introduce a graph modeling of problem~(\ref{eq:dual_problem}).

\subsection{Graph Model}~\label{subsec:graph}
Let $G$ be a directed graph $G=(V,E,s,t)$, where $V$ is a set of vertices,
$E\subseteq V\times V$ a set of arcs, $s$ a source, and $t$ a sink. Let $c$ and
$c'$ be two functions on the arcs, $c: E \to \Real$ and $c': E \to \Real^+$, where
$c$ is a \emph{cost function} and $c'$ is a non-negative \emph{capacity
function}.  A \emph{flow} is a non-negative function on arcs that satisfies
capacity constraints on all arcs (the value of the flow on an arc is less than or equal to
the arc capacity) and conservation constraints on all vertices (the sum of
incoming flows at a vertex is equal to the sum of outgoing flows) except for
the source and the sink.

We introduce a \emph{canonical} graph $G$ associated with our optimization problem, and uniquely characterized by the following construction: \\
~(i) $V$ is the union of two sets of vertices $V_u$ and $V_{gr}$, where $V_u$ contains exactly one vertex for each index $j$ in
$\IntSet{p}$, and $V_{gr}$ contains exactly one vertex for each group~$g$ in~$\GG$. We thus have $|V|=|\GG|+p$. For simplicity, we
identify groups and indices with the vertices of the graph.  \\
~(ii) For every group $g$ in $\GG$, $E$ contains an arc $(s,g)$.
These arcs have capacity $\lambda\eta_g$ and zero cost. \\
~(iii) For every group $g$ in $\GG$, and every index $j$ in $g$, $E$
contains an arc $(g,j)$ with zero cost and infinite capacity. We
denote by $\xib_j^g$ the flow on this arc. \\
~(iv) For every index $j$ in $\IntSet{p}$, $E$ contains an arc $(j,t)$
with infinite capacity and a cost~\mbox{$c_j\!\defin\!\frac{1}{2}(\u_j-\xibbar_j)^2$}, where $\xibbar_j$ is the flow on $(j,t)$.
Note that by flow conservation, we necessarily have $\xibbar_j\!=\!\sum_{g \in \GG}{\xib_j^g}$.

Examples of canonical graphs are given in Figures \ref{subfig:grapha}-\subref{subfig:graphc}.
The flows $\xib_j^g$ associated with $G$ can now be identified with 
the variables of problem (\ref{eq:dual_problem}): indeed, the sum of the costs on the edges leading to the sink is equal
to the objective function of (\ref{eq:dual_problem}), while the capacities of the arcs $(s,g)$ match the constraints on each group. This shows that finding a flow \emph{minimizing the sum of the costs} on such a graph is
equivalent to solving problem~(\ref{eq:dual_problem}). 

When some groups are included in others, the canonical graph can be simplified to yield a graph with a smaller number of edges. Specifically,
if $h$ and $g$ are groups with $h \subset g$, the edges $(g,j)$ for $j \in h$ carrying a flow $\xib^g_j$ can be removed and replaced by a single edge $(g,h)$ of infinite capacity and zero cost, carrying the flow $\sum_{j \in h} \xib^g_j$. This simplification is illustrated in Figure~\ref{subfig:graphd}, with a 
graph equivalent to the one of Figure~\ref{subfig:graphc}.
This does not change
the optimal value of~$\xibbar^\star$, which is the quantity of interest for
computing the optimal primal variable~$\w^\star$. 
We present in Appendix~\ref{appendix:equivalent} a formal definition
of equivalent graphs.
These simplifications are useful in practice, since they reduce the
number of edges in the graph and improve the speed of the algorithms we are now
going to present.
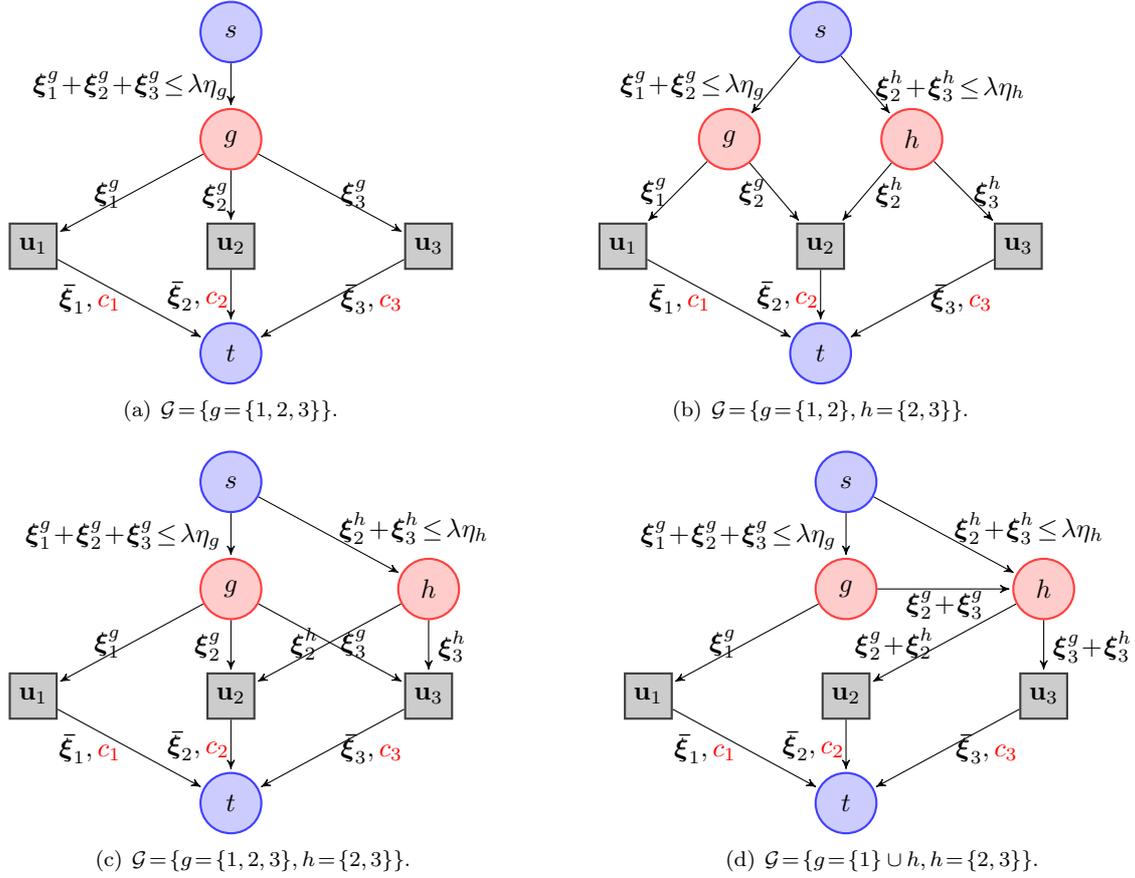
\begin{figure}[hbtp!]
\tikzstyle{source}=[circle,thick,draw=blue!75,fill=blue!20,minimum size=8mm]
\tikzstyle{sink}=[circle,thick,draw=blue!75,fill=blue!20,minimum size=8mm]
\tikzstyle{group}=[place,thick,draw=red!75,fill=red!20, minimum size=8mm]
\tikzstyle{var}=[rectangle,thick,draw=black!75,fill=black!20,minimum size=6mm]
\def\distnode{1.42cm}
\def\distnodex{2.6cm}
\tikzstyle{every label}=[red]
   \begin{center}
      \subfigure[$\GG\!=\!\{ g\!=\!\{1,2,3\} \}$.]{
      \begin{tikzpicture}[node distance=\distnode,>=stealth',bend angle=45,auto]
         \begin{scope}
            \node [source]   (s)                                    {$s$};
            \node [group]    (g1)  [below of=s]                      {$g$}
            edge  [pre] node[left,xshift=1mm] {$\xib^g_1 \!+\! \xib^g_2 \!+\! \xib^g_3 \!\leq\! \lambda \eta_g$} (s);
            \node [var] (u2) [below of=g1]                    {$\u_2$}
            edge  [pre] node[above, left,xshift=1mm] {$\xib^{g}_2$} (g1);
            \node [var] (u1)  [left of=u2, node distance=\distnodex] {$\u_1$}
            edge  [pre] node[above, left] {$\xib^{g}_1$} (g1);
            \node [var] (u3) [right of=u2, node distance=\distnodex] {$\u_3$}
            edge  [pre] node[above, right] {$\xib^{g}_3$} (g1);
            \node [sink] (si) [below of=u2] {$t$}
            edge [pre] node[above,left] {$\xibbar_1, \color{red} c_1$} (u1)
            edge [pre] node[above,left,xshift=1mm] {$\xibbar_2,\color{red} c_2$} (u2)
            edge [pre] node[above,right] {$\xibbar_3,\color{red} c_3$} (u3);
         \end{scope}
      \end{tikzpicture}\label{subfig:grapha}
      } \hfill 
      \subfigure[$\GG\!=\!\{ g\!=\!\{1,2\},h\!=\!\{2,3\} \}$.]{
      \begin{tikzpicture}[node distance=\distnode,>=stealth',bend angle=45,auto]
         \begin{scope}
            \node [source]   (s)                                    {$s$};
            \node [group]    (g)  [below of=s,xshift=-12mm]                      {$g$}
            edge  [pre] node[left] {$\xib^g_1 \!+ \!\xib^g_2 \!\leq \!\lambda \eta_g$} (s);
            \node [group]    (h)  [below of=s,xshift=12mm]                      {$h$}
            edge  [pre] node[right] {$\xib^h_2 \!+ \!\xib^h_3 \!\leq \!\lambda \eta_h$} (s);
            \node [var] (u2) [below of=g,xshift=12mm]                    {$\u_2$}
            edge  [pre] node[above, right] {$\xib^{h}_2$} (h)
            edge  [pre] node[above, left] {$\xib^{g}_2$} (g);
            \node [var] (u1)  [left of=u2, node distance=\distnodex] {$\u_1$}
            edge  [pre] node[above, left] {$\xib^{g}_1$} (g);
            \node [var] (u3) [right of=u2, node distance=\distnodex] {$\u_3$}
            edge  [pre] node[above, right] {$\xib^{h}_3$} (h);
            \node [sink] (si) [below of=u2] {$t$}
            edge [pre] node[above,left] {$\xibbar_1, \color{red} c_1$} (u1)
            edge [pre] node[above,left,xshift=1mm] {$\xibbar_2,\color{red} c_2$} (u2)
            edge [pre] node[above,right] {$\xibbar_3,\color{red} c_3$} (u3);
         \end{scope}
      \end{tikzpicture}\label{subfig:graphb}
      } ~~~~~~~~~~~\\
      \subfigure[$\GG\!=\!\{ g\!=\!\{1,2,3\},h\!=\!\{2,3\} \}$.]{
      \begin{tikzpicture}[node distance=\distnode,>=stealth',bend angle=45,auto]
         \begin{scope}
            \node [source]   (s)                                    {$s$};
            \node [group]    (g)  [below of=s]                      {$g$}
            edge  [pre] node[left] {$\xib^g_1 \!+ \!\xib^g_2 \!+ \!\xib^g_3 \!\leq \!\lambda \eta_g$} (s);
            \node [group]    (h)  [right of=g,node distance=\distnodex]                      {$h$}
            edge  [pre] node[right,yshift=1mm] {$\xib^h_2 \!+ \!\xib^h_3 \!\leq \!\lambda \eta_h$} (s);
            \node [var] (u2) [below of=g]                    {$\u_2$}
            edge  [pre] node[above, left] {$\xib^{h}_2$} (h)
            edge  [pre] node[above, left] {$\xib^{g}_2$} (g);
            \node [var] (u1)  [left of=u2, node distance=\distnodex] {$\u_1$}
            edge  [pre] node[above, left] {$\xib^{g}_1$} (g);
            \node [var] (u3) [right of=u2, node distance=\distnodex] {$\u_3$}
            edge  [pre] node[above, right] {$\xib^{g}_3$} (g)
            edge  [pre] node[above, right] {$\xib^{h}_3$} (h);
            \node [sink] (si) [below of=u2] {$t$}
            edge [pre] node[above,left] {$\xibbar_1, \color{red} c_1$} (u1)
            edge [pre] node[above,left,xshift=1mm] {$\xibbar_2,\color{red} c_2$} (u2)
            edge [pre] node[above,right] {$\xibbar_3,\color{red} c_3$} (u3);
         \end{scope}
      \end{tikzpicture}\label{subfig:graphc}
      } \hfill
      \subfigure[$\GG\!=\!\{ g\!=\!\{1\}\cup h,h\!=\!\{2,3\} \}$.]{
      \begin{tikzpicture}[node distance=\distnode,>=stealth',bend angle=45,auto]
         \begin{scope}
            \node [source]   (s)                                    {$s$};
            \node [group]    (g)  [below of=s]                      {$g$}
            edge  [pre] node[left] {$\xib^g_1 \!+ \!\xib^g_2 \!+ \!\xib^g_3 \!\leq \!\lambda \eta_g$} (s);
            \node [group]    (h)  [right of=g, node distance=\distnodex]                      {$h$}
            edge  [pre] node[right,yshift=1mm] {$\xib^h_2 \!+ \!\xib^h_3 \!\leq \!\lambda \eta_h$} (s)
            edge  [pre] node[below,yshift=1mm] {$\xib^g_2 \!+ \!\xib^g_3$} (g);
            \node [var] (u2) [below of=g]                    {$\u_2$}
            edge  [pre] node[above,left] {$\xib^g_2\!+\!\xib^{h}_2$} (h);
            \node [var] (u1)  [left of=u2, node distance=\distnodex] {$\u_1$}
            edge  [pre] node[above, left] {$\xib^{g}_1$} (g);
            \node [var] (u3) [right of=u2, node distance=\distnodex] {$\u_3$}
            edge  [pre] node[above, right] {$\xib^g_3\!+\!\xib^{h}_3$} (h);
            \node [sink] (si) [below of=u2] {$t$}
            edge [pre] node[above,left] {$\xibbar_1, \color{red} c_1$} (u1)
            edge [pre] node[above,left,xshift=1mm] {$\xibbar_2,\color{red} c_2$} (u2)
            edge [pre] node[above,right] {$\xibbar_3,\color{red} c_3$} (u3);
         \end{scope}
      \end{tikzpicture} \label{subfig:graphd}
      } 
   \end{center}
   \caption{Graph representation of simple proximal problems with different
   group structures $\GG$. The three indices $1,2,3$ 
   are represented as grey squares, and the groups
   $g,h$ in $\GG$ as red discs. The source is
   linked to every group $g,h$ with respective maximum capacity
   $\lambda\eta_g,\lambda\eta_h$ and zero cost. Each variable $\u_j$ is linked to the sink
   $t$, with an infinite capacity, and with a cost
   $c_j\!\defin\!\frac{1}{2}(\u_j-\xibbar_j)^2$. All other arcs in the graph have
   zero cost and infinite capacity. They represent inclusion relations in-between groups, and between groups and variables.
   The graphs
   \subref{subfig:graphc} and \subref{subfig:graphd} correspond to a special case of
   tree-structured hierarchy in the sense of \cite{jenatton3}. Their min-cost
   flow problems are equivalent.} \label{fig:graphs}
\end{figure}
\subsection{Computation of the Proximal Operator}\label{subsec:prox}
Quadratic min-cost flow problems have been well studied in the operations
research literature~\cite{hochbaum}. One of the simplest
cases, where~$\GG$ contains a single group $g$ as in Figure~\ref{subfig:grapha},
can be solved by an orthogonal projection on the $\ell_1$-ball of
radius~$\lambda\eta_g$. 
It has been shown, both in machine learning~\cite{duchi} and operations research~\cite{hochbaum,brucker}, that such a projection can be done in $O(p)$ operations. 
When the group structure is a tree 
as in Figure~\ref{subfig:graphd}, strategies developed in
the two communities are also similar~\cite{jenatton3,hochbaum}, and solve the
problem in $O(p d)$ operations, where $d$ is the depth of the tree.

The general case of overlapping groups is more difficult. Hochbaum and Hong have shown
in~\cite{hochbaum} that \emph{quadratic min-cost flow problems} can be reduced to a specific
\emph{parametric max-flow} problem, for which an efficient algorithm
exists~\cite{gallo}.\footnote{By definition, a parametric max-flow problem
consists in solving, for every value of a parameter, a max-flow problem on a
graph whose arc capacities depend on this parameter.} While this approach
could be used to solve Eq.~(\ref{eq:dual_problem}), it ignores the fact that
our graphs have non-zero costs only on edges leading to the sink. 
To take advantage of this specificity, we propose the dedicated Algorithm~\ref{algo:prox}. 
Our method clearly shares some similarities with a simplified version of~\cite{gallo} presented in~\cite{babenko}, namely a 
divide and conquer strategy. Nonetheless, we performed an empirical comparison described in Appendix~\ref{appendix:exp}, 
which shows that our dedicated algorithm has significantly better performance in practice.
\begin{algorithm}[hbtp]
\caption{Computation of the proximal operator for overlapping groups.}\label{algo:prox}
\begin{algorithmic}[1]
\INPUT $\u \in \R{p}$, a set of groups $\GG$, positive weights
$(\eta_g)_{g\in\GG}$, and $\lambda$ (regularization parameter).
\STATE Build the initial graph $G_0=(V_0,E_0,s,t)$ as explained in Section \ref{subsec:prox}.
\STATE Compute the optimal flow: $\xibbar \leftarrow \text{\texttt{computeFlow}}(V_0,E_0)$.
\STATE {\bf{Return:}} $\w = \u-\xibbar$ (optimal solution of the proximal problem).
\end{algorithmic}
\vspace*{0.1cm}
{\bf Function} \texttt{computeFlow}($V=V_u \cup V_{gr},E$)
\begin{algorithmic}[1]
\STATE Projection step: $\gammab \leftarrow
\argmin_\gammab \sum_{j \in V_u}\frac{1}{2}(\u_j -\gammab_j)^2 \st  \sum_{j \in V_u} \gammab_j \leq \lambda\sum_{g \in V_{gr}}\eta_g.$
\STATE For all nodes $j$ in $V_u$, set $\gammab_j$ to be the capacity of the arc $(j,t)$.
\STATE Max-flow step: Update $(\xibbar_j)_{j \in V_u}$ by computing a max-flow on the graph $(V,E,s,t)$.
\IF{ $\exists~j \in V_u \st \xibbar_j \neq \gammab_j$}
\STATE Denote by $(s,V^+)$ and $(V^-,t)$ the two disjoint subsets of $(V,s,t)$ separated by
the minimum $(s,t)$-cut of the graph, and remove the arcs between $V^+$ and $V^-$. 
Call $E^+$ and $E^-$ the two remaining disjoint subsets of~$E$ corresponding to $V^+$ and $V^-$.
\STATE $(\xibbar_j)_{j \in V_u^+} \leftarrow \text{\texttt{computeFlow}}(V^+,E^+)$.
\STATE $(\xibbar_j)_{j \in V_u^-} \leftarrow \text{\texttt{computeFlow}}(V^-,E^-)$.
\ENDIF
\STATE {\bf{Return:}} $(\xibbar_j)_{j \in V_u}$.
\end{algorithmic}
\end{algorithm}
Informally, \texttt{computeFlow}$(V_0,E_0)$ returns the optimal flow vector $\xibbar$,
proceeding as follows: This function first solves a relaxed version of problem
Eq.~(\ref{eq:dual_problem}) obtained by replacing the sum of the vectors~$\xib^g$ by a
single vector $\gammab$ whose $\ell_1$-norm
should be less than, or equal to, the sum of the constraints on the
vectors~$\xib^g$. The optimal vector $\gammab$ therefore gives a lower bound
$||\u-\gammab||_2^2/2$ on the optimal cost.  Then, the maximum-flow step~\cite{goldberg} tries to find a feasible flow such that
the vector $\xibbar$ matches $\gammab$.  If~$\xibbar=\gammab$, then the cost of the flow
reaches the lower bound, and the flow is optimal.
If $\xibbar \neq \gammab$, the lower bound cannot be reached, and we construct a
minimum $(s,t)$-cut of the graph~\cite{ford} that defines two disjoints sets of
nodes $V^+$ and $V^-$; $V^+$ is the part of the graph that can potentially 
receive more flow from the source, whereas all arcs linking $s$ to $V^-$ are 
saturated. The properties of a min $(s,t)$-cut~\cite{bertsekas2} imply that there are
no arcs from $V^+$ to $V^-$ (arcs inside $V$ have infinite capacity by
construction), and that there is no flow on arcs from~$V^-$ to~$V^+$.
At this point, it is possible to show that the value of the optimal min-cost
flow on these arcs is also zero. Thus, removing them yields an equivalent optimization problem, 
which can be decomposed into
two independent problems of smaller size and 
solved recursively by the calls to \texttt{computeFlow}$(V^+,E^+)$ and
\texttt{computeFlow}$(V^-,E^-)$. 
Note that when $\Omega$ is the $\ell_1$-norm, our algorithm solves
problem~(\ref{eq:dual_problem}) during the first projection step in line $1$
and stops.
A formal proof of correctness of Algorithm~\ref{algo:prox}
and further details are relegated to Appendix~\ref{appendix:convergence}.

The approach of \cite{hochbaum,gallo} is
guaranteed to have the same worst-case complexity as a single max-flow
algorithm. However, we have experimentally observed a significant discrepancy
between the worst case and empirical complexities for these flow problems,
essentially because the empirical cost of each max-flow is significantly smaller
than its theoretical cost.  Despite the fact that the worst-case guarantee of our algorithm 
is weaker than their (up to a factor $|V|$), it is more adapted to the
structure of our graphs and has proven to be much faster in our experiments
(see supplementary material).

Some implementation details are crucial to the efficiency of the algorithm: 
\begin{itemize}
   \item \textbf{Exploiting maximal connected components}: When there exists no
arc between two subsets of $V$, it is possible to process them independently 
to solve the global min-cost flow problem. 
To that effect, before calling
the function \texttt{computeFlow}($V,E$), we look for maximal connected components
$(V_1,E_1),\ldots,(V_N,E_N)$ and call sequentially the procedure
\texttt{computeFlow}($V_i,E_i$) for $i$ in $\IntSet{N}$. 
\item  \textbf{Efficient max-flow algorithm}: We have implemented the
``push-relabel'' algorithm of \cite{goldberg} to solve our max-flow
problems, using classical heuristics that significantly speed it up in practice
(see \cite{goldberg,cherkassky}).
 Our implementation uses the so-called ``highest-active vertex
selection rule, global and gap heuristics'' (see \cite{goldberg,cherkassky}),
and has a worst-case complexity of $O(|V|^2 |E|^{1/2})$ for a graph
$(V,E,s,t)$.
This algorithm leverages the concept of \emph{pre-flow} that relaxes the
definition of flow and allows vertices to have a positive excess.
\item \textbf{Using flow warm-restarts}: 
   Our algorithm can be initialized with any valid pre-flow, enabling warm-restarts when
the max-flow is called several times as in our algorithm.
 \item \textbf{Improved projection step}:
The first line of the procedure \texttt{computeFlow} can be replaced by 
$\gammab \leftarrow \argmin_\gammab \sum_{j \in V_u}\frac{1}{2}(\u_j
-\gammab_j)^2 \st  \sum_{j \in V_u} \gammab_j \leq \lambda\sum_{g \in V_{gr}}\eta_g ~\text{and}~
|\gammab_j| \leq \lambda\sum_{g \ni j}\eta_g.$
The idea is that the structure of the graph will not allow $\xibbar_j$ to be greater
than $\lambda\sum_{g \ni j}\eta_g$ after the max-flow step. 
Adding these additional constraints leads to better performance when the graph
is not well balanced.  This modified
projection step can still be computed in linear time~\cite{brucker}.
\end{itemize}

\subsection{Computation of the Dual Norm}\label{subsec:dual}
The dual norm $\Omega^*$ of $\Omega$, defined for any vector $\kappab$ in $\R{p}$ by 
$\Omega^*(\kappab)\defin\max_{\Omega(\z)\leq 1}\z^\top\kappab$, 
is a key quantity to study sparsity-inducing regularizations \cite{jenatton, hkl, negahban2009unified}.
We use it here to monitor the convergence of the proximal method through a duality gap,
and define a proper optimality criterion for problem~(\ref{eq:formulation}).
We denote by $f^*$ the Fenchel conjugate of $f$~\cite{borwein}, defined by $f^*(\kappab)\defin\sup_{\z} [\z^\top\kappab - f(\z)]$.
The duality gap for problem~(\ref{eq:formulation}) can be derived from standard Fenchel duality arguments \cite{borwein} and it is equal to
$
 f(\w)+ \lambda \Omega(\w) + f^*(-\kappab)\ \text{for}\ \w,\kappab\ \text{in}\ \R{p}\ \text{with}\
\Omega^*(\kappab)\leq \lambda
$.
Therefore, evaluating the duality gap requires to compute efficiently $\Omega^*$
in order to find a feasible dual variable~$\kappab$.
This is equivalent to solving another network flow problem, based on the following variational formulation:
\begin{equation}
\Omega^*(\kappab) = \!\!\! \min_{\xib\in\RR{p}{|\G|}} \!\!\! \tau \quad \text{s.t.}\quad \sum_{g\in\G}\xib^g=\kappab,\ \text{and}\ \forall g\in\G,\ \|\xib^g\|_1 \leq \tau\eta_g ~~\text{with}~~\
\xib_j^g=0\ \text{if}\ j\notin g. \label{eq:dual_norm}
\end{equation}
In the network problem associated with~(\ref{eq:dual_norm}), the capacities on the arcs $(s,g)$,
$g \in \GG$, are set to $\tau\eta_g$, and the capacities on the arcs $(j,t)$, $j$ in
$\IntSet{p}$, are fixed to $\kappab_j$.
Solving problem~(\ref{eq:dual_norm}) amounts to finding the smallest value of $\tau$, such
that there exists a flow saturating the capacities $\kappab_j$ on the arcs
leading to the sink~$t$ (i.e., $\xibbar=\kappab$). 
Equration~(\ref{eq:dual_norm}) and the algorithm below 
are proven to be correct in Appendix~\ref{appendix:convergence}.
\begin{algorithm}[!hbtp]
\caption{Computation of the dual norm.}\label{algo:dual_norm}
\begin{algorithmic}[1]
\INPUT $\kappab \in \R{p}$, a set of groups $\GG$, positive weights
$(\eta_g)_{g\in\GG}$.
\STATE Build the initial graph $G_0=(V_0,E_0,s,t)$ as explained in Section~\ref{subsec:dual}.
\STATE $\tau \leftarrow \text{\texttt{dualNorm}}(V_0,E_0)$.
\STATE {\bf{Return:}} $\tau$ (value of the dual norm).
\end{algorithmic}
\vspace*{0.1cm}
{\bf Function} \texttt{dualNorm}($V = V_u \cup V_{gr},E$)
\begin{algorithmic}[1]
\STATE $\tau\! \leftarrow \!(\sum_{j \in V_u} \kappab_j) / (\sum_{g \in V_{gr}} \eta_g)$ and set the capacities of arcs $(s,g)$ to $\tau\eta_g$ for all $g$ in~$V_{gr}$.
\STATE Max-flow step: Update $(\xibbar_j)_{j \in V_u}$ by computing a max-flow on the graph $(V,E,s,t)$.
\IF{ $\exists~j \in V_u \st \xibbar_j \neq \kappab_j$}
\STATE Define $(V^+,E^+)$ and $(V^-,E^-)$ as in Algorithm~\ref{algo:prox}, and set $\tau \leftarrow \text{\texttt{dualNorm}}(V^-,E^-)$.
\ENDIF
\STATE {\bf {Return:} } $\tau$.
\end{algorithmic}
\end{algorithm}

\section{Applications and Experiments} \label{sec:exp}
Our experiments use the algorithm of~\cite{beck} based on our proximal operator, with weights $\eta_g$ set to~$1$. We present this algorithm in more details in Appendix~\ref{appendix:fista}.
\subsection{Speed Comparison}
We compare our method (ProxFlow) and two generic optimization techniques,
namely a subgradient descent (SG) and an interior point
method,\footnote{In our simulations, we use the commercial software
\texttt{Mosek},  \texttt{http://www.mosek.com/}} on a regularized linear
regression problem.  Both SG and ProxFlow are implemented in \texttt{C++}.
Experiments are run on a single-core $2.8$ GHz CPU.  We consider a design
matrix $\X$ in $\Real^{n \times p}$ built from overcomplete dictionaries
of discrete cosine transforms (DCT), which are naturally organized on one-
or two-dimensional grids and display local correlations.  The following
families of groups $\G$ using this spatial information are thus
considered: (1) every contiguous sequence of length $3$ for the
one-dimensional case, and (2) every $3\!\times\!3$-square in the
two-dimensional setting.  We generate vectors $\y$ in $\R{n}$ according to
the linear model $\y = \X\w_0 + \varepsilonb$, 
where $\varepsilonb \sim \N(0,0.01\|\X\w_0\|_2^2)$.  The vector $\w_0$ has
about $20\%$ percent nonzero components, randomly selected, while
respecting the structure of $\G$, and uniformly generated between
$[-1,1]$.

In our experiments, the regularization parameter $\lambda$ is chosen to
achieve this level of sparsity.  For SG, we take the step size to be equal
to $a/(k+b)$, where $k$ is the iteration number, and $(a,b)$ are the best
parameters selected in $\{10^{-3},\dots,10\}\!\times\!\{10^2,10^3,10^4\}$.
For the interior point methods, since problem~(\ref{eq:formulation}) can
be cast either as a quadratic (QP) or as a conic program (CP), we show in
Figure~\ref{fig:speed_cmp} the results for both formulations.  Our
approach compares favorably with the other methods, on three problems of
different sizes, $(n,p)\in\{(100,10^3),(1024,10^4),(1024,10^5)\}$, see
Figure~\ref{fig:speed_cmp}.  In addition, note that QP, CP and SG do not
obtain sparse solutions, whereas ProxFlow does.  We have also run ProxFlow
and SG on a larger dataset with $(n,p)=(100,10^6)$: after $12$ hours,
ProxFlow and SG have reached a relative duality gap of $0.0006$ and $0.02$
respectively.\footnote{Due to the computational burden, QP and CP could
not be run on every problem.}
\begin{figure}[hbtp]
    \centering
   \includegraphics[width=0.31\textwidth]{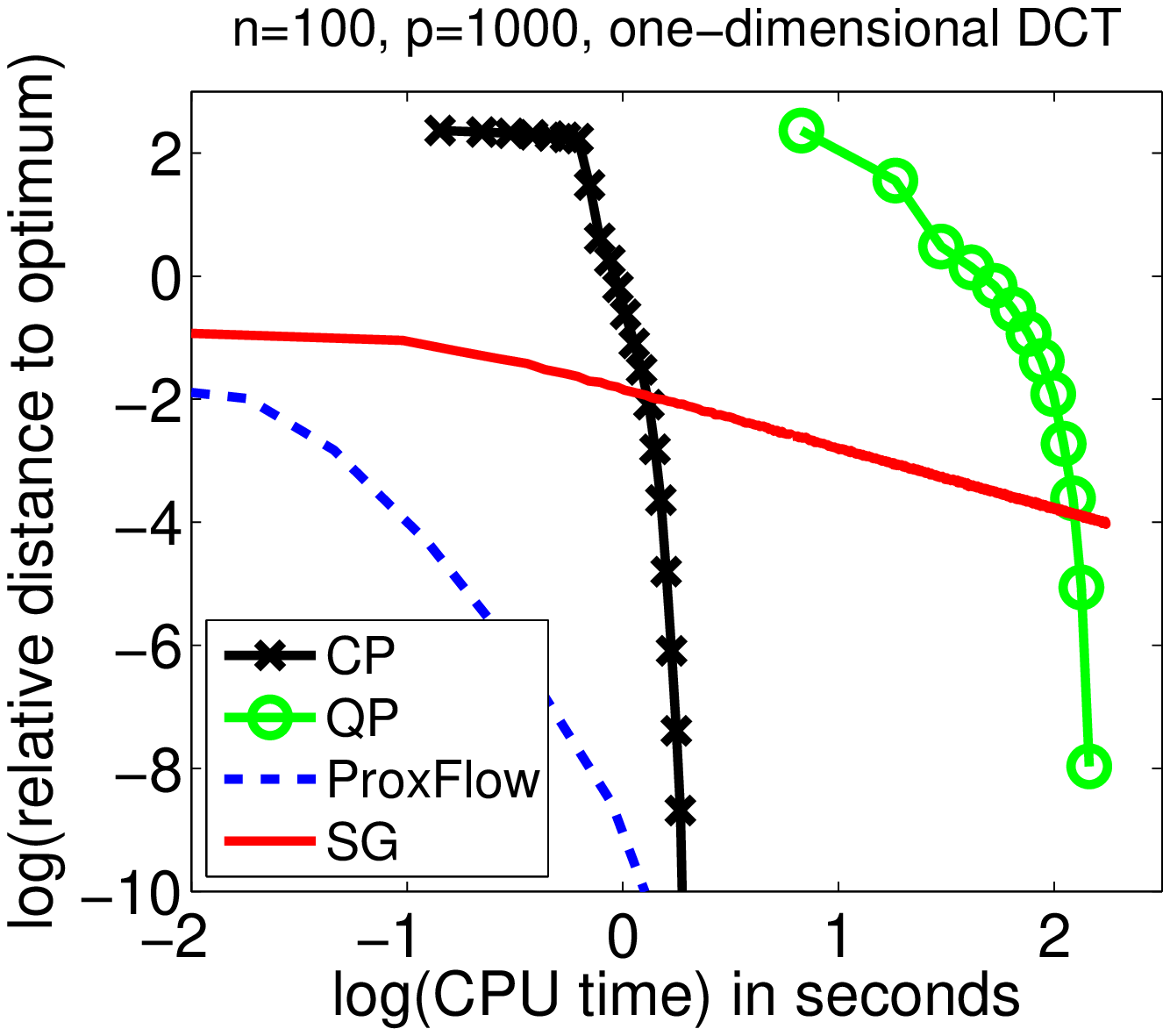} \hfill
   \includegraphics[width=0.31\textwidth]{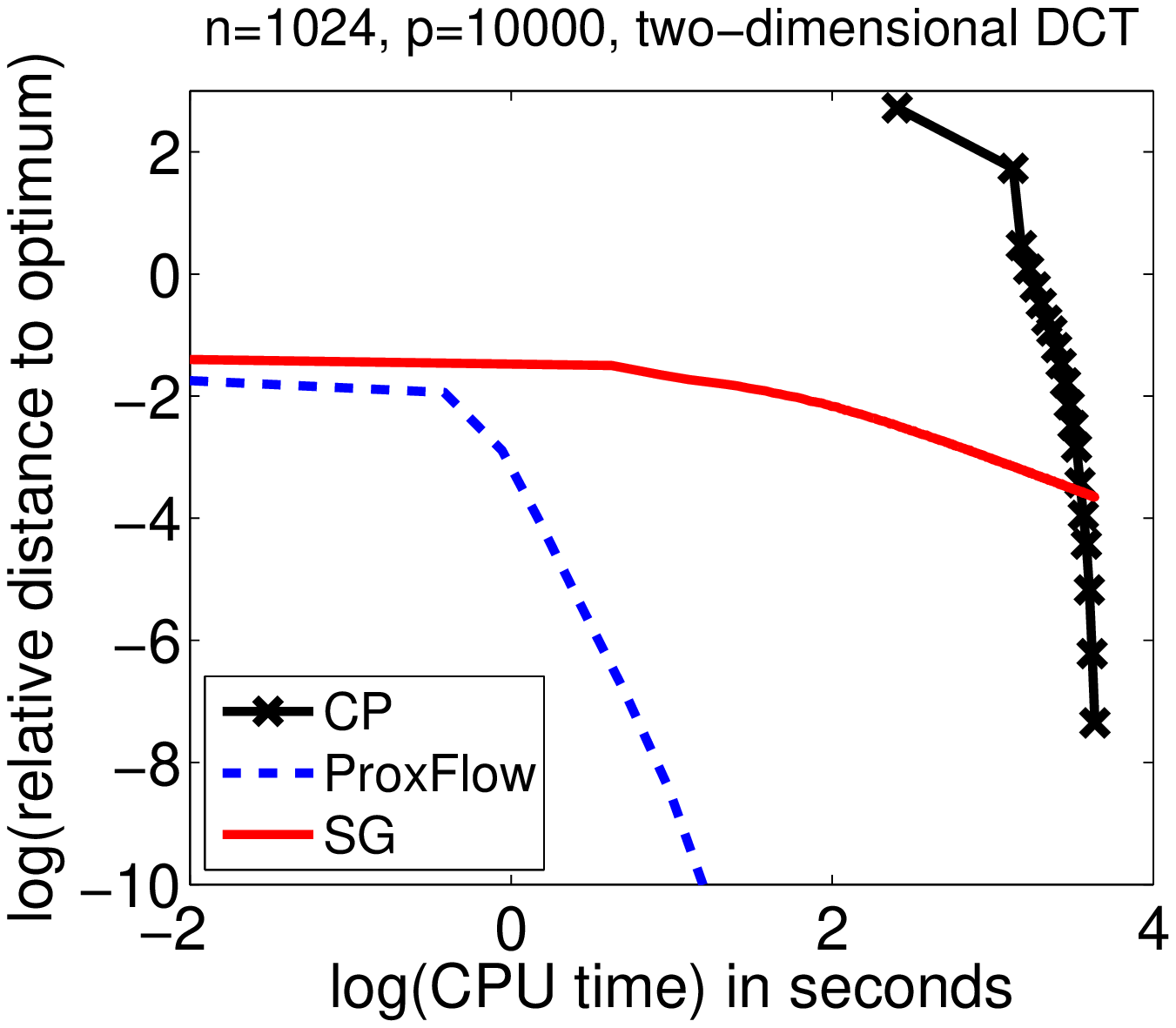} \hfill
   \includegraphics[width=0.31\textwidth]{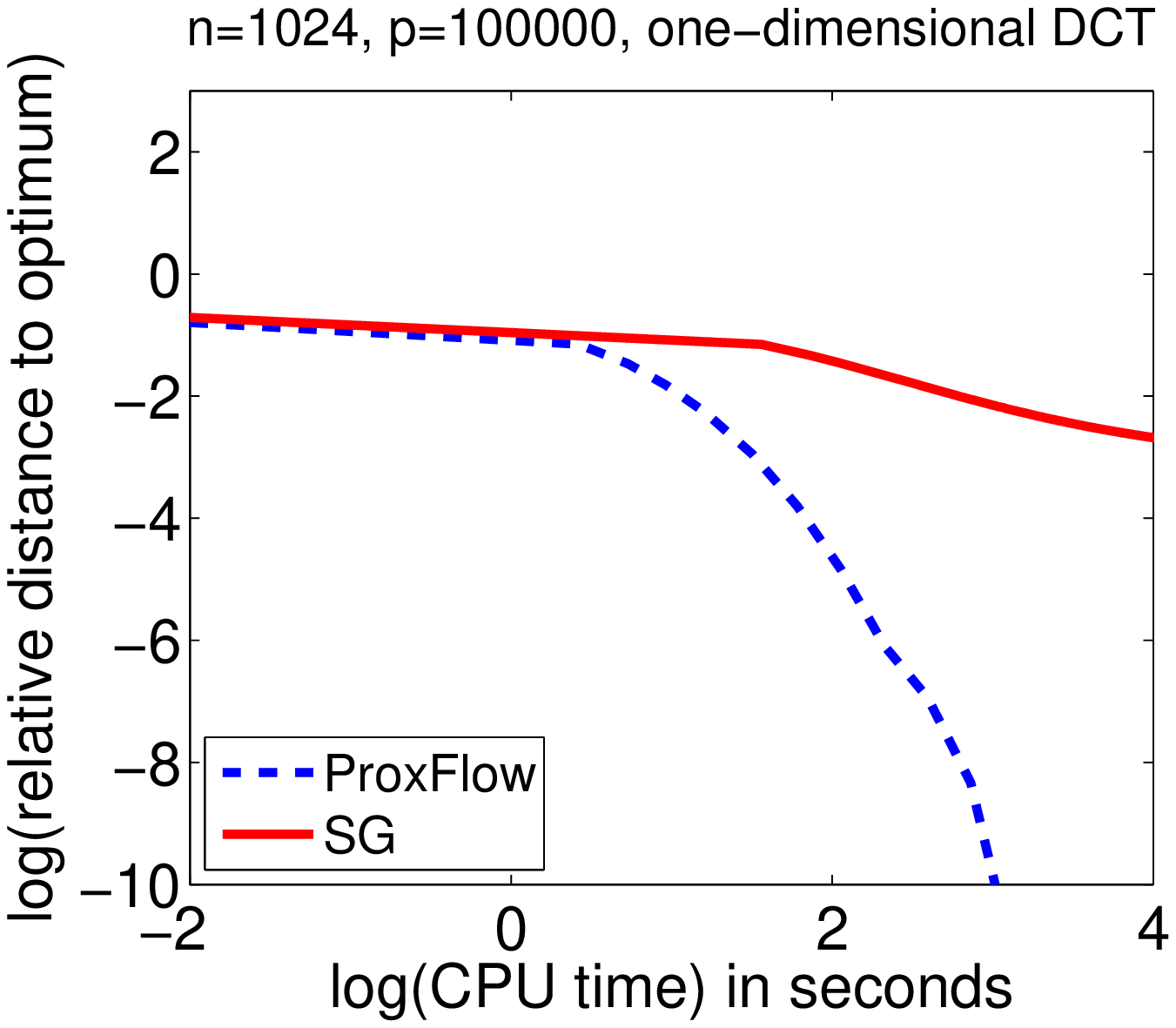} 
   \caption{Speed comparisons: distance to the optimal primal value versus CPU time (log-log scale). Due to the computational burden, QP and CP could not be run on every problem.} 
\label{fig:speed_cmp}
\end{figure}
\subsection{Background Subtraction}
Following \cite{huang}, we consider a background subtraction task. Given a sequence of frames from a fixed camera,
we try to segment out foreground objects in a new image. 
If we denote by $\y\in\R{n}$ this image composed of $n$ pixels, 
we model $\y$ as a sparse linear combination of $p$ other images $\X\in\RR{n}{p}$, plus an error term $\e$ in $\R{n}$,
i.e., $\y \approx \X \w + \e$ for some sparse vector $\w$ in $\R{p}$. 
This approach is reminiscent of \cite{yima} in the context of face recognition, 
where $\e$ is further made sparse to deal with small occlusions.
The term $\X \w$ accounts for \textit{background} parts present in both~$\y$ and~$\X$, while $\e$ contains specific, 
or \textit{foreground}, objects in $\y$.
The resulting optimization problem is
$
\min_{\w,\e} \frac{1}{2} \|\y \! - \! \X\w \! - \! \e\|_2^2 + \lambda_1\|\w\|_1 + \lambda_2\|\e\|_1,\ \text{with}\ \lambda_1,\lambda_2 \geq 0.
$
In this formulation, the $\ell_1$-norm penalty on $\e$ does not take into account the fact that
neighboring pixels in $\y$ are likely to share the same label (background or foreground), 
which may lead to scattered pieces of foreground and background regions (Figure~\ref{fig:background_sub}).
We therefore put an additional structured regularization term $\Omega$ on $\e$, where the groups in $\G$ are all the overlapping $3\!\times\!3$-squares on the image.
A dataset with hand-segmented evaluation images is used to illustrate
the effect of $\Omega$.\footnote{{\scriptsize\texttt{http://research.microsoft.com/en-us/um/people/jckrumm/wallflower/testimages.htm}}}
For simplicity, we use a single regularization parameter, i.e., $\lambda_1=\lambda_2$, chosen to maximize the number of pixels matching the ground truth.
We consider $p=200$ images with $n=57600$ pixels (i.e., a resolution of $120\!\times\!160$, times 3 for the RGB channels).
As shown in Figure~\ref{fig:background_sub}, adding $\Omega$ improves the background subtraction results for the two tested images, 
by removing the scattered artifacts due to the lack of structural constraints of the $\ell_1$-norm, which encodes neither spatial nor color consistency. 
 \begin{figure}[hbtp]
 \begin{center}
 \begin{tabular}{ccccc}
  \includegraphics[width=0.19\textwidth]{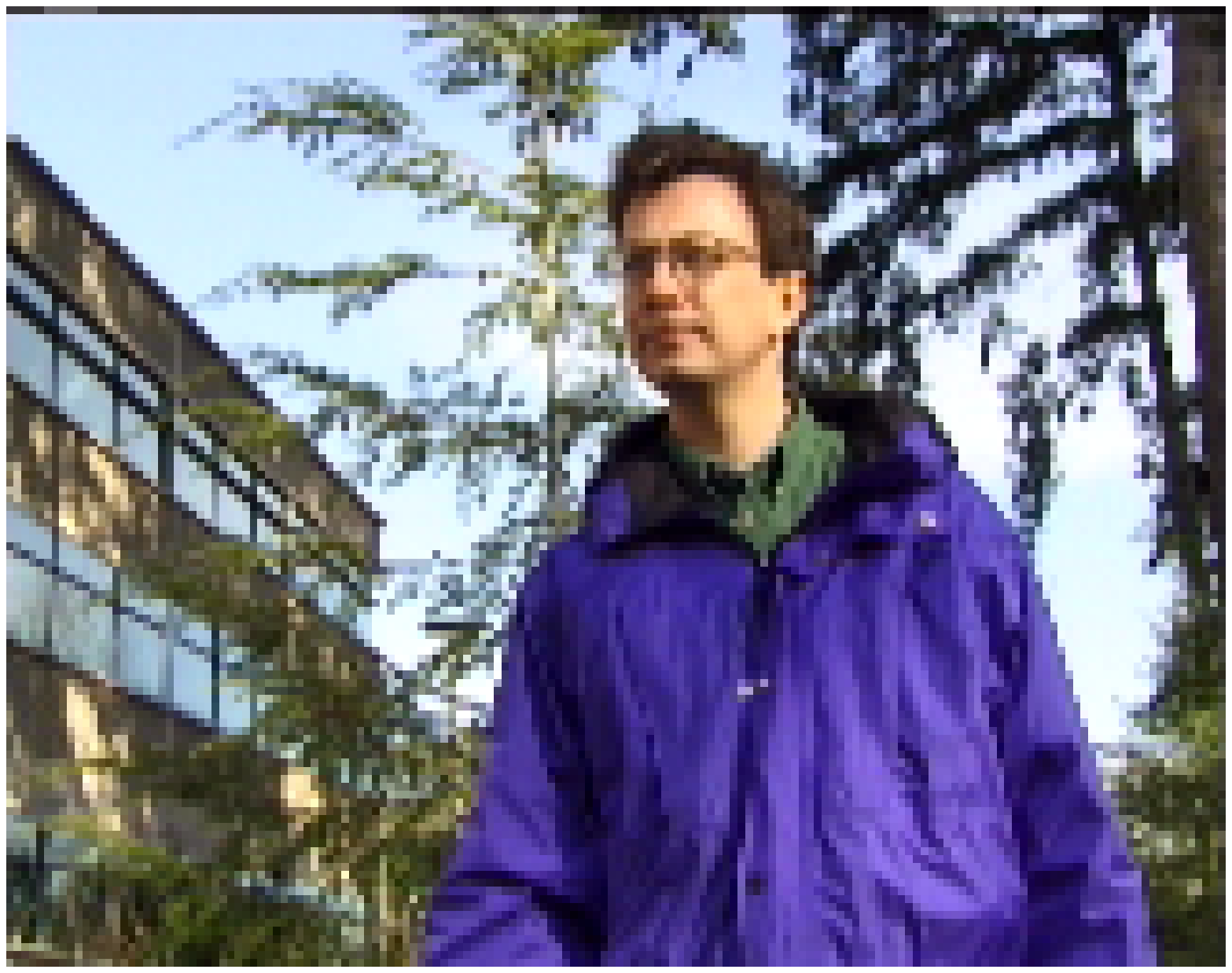} & \EcartTabFig
  \includegraphics[width=0.19\textwidth]{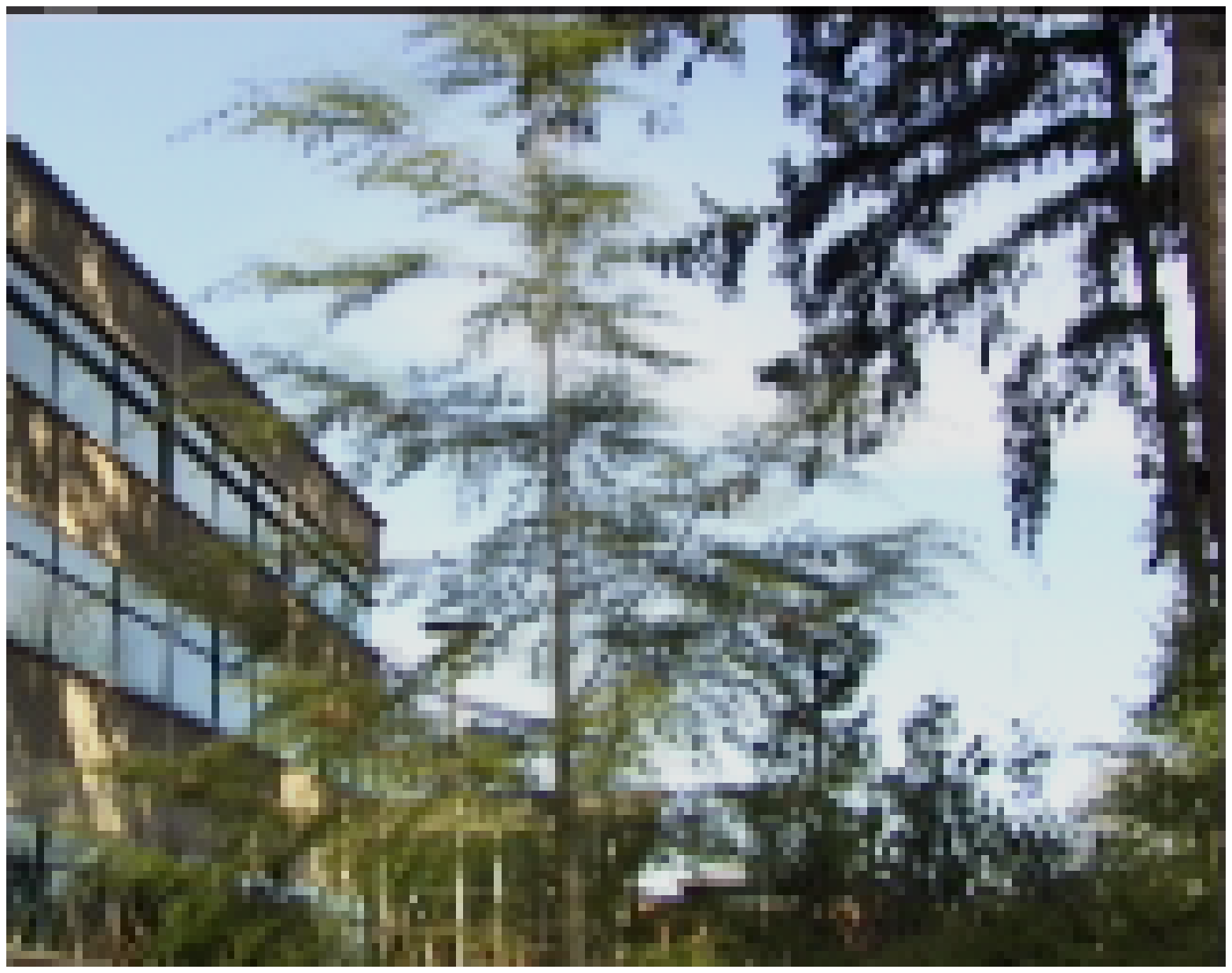} & \EcartTabFig
  \includegraphics[width=0.19\textwidth]{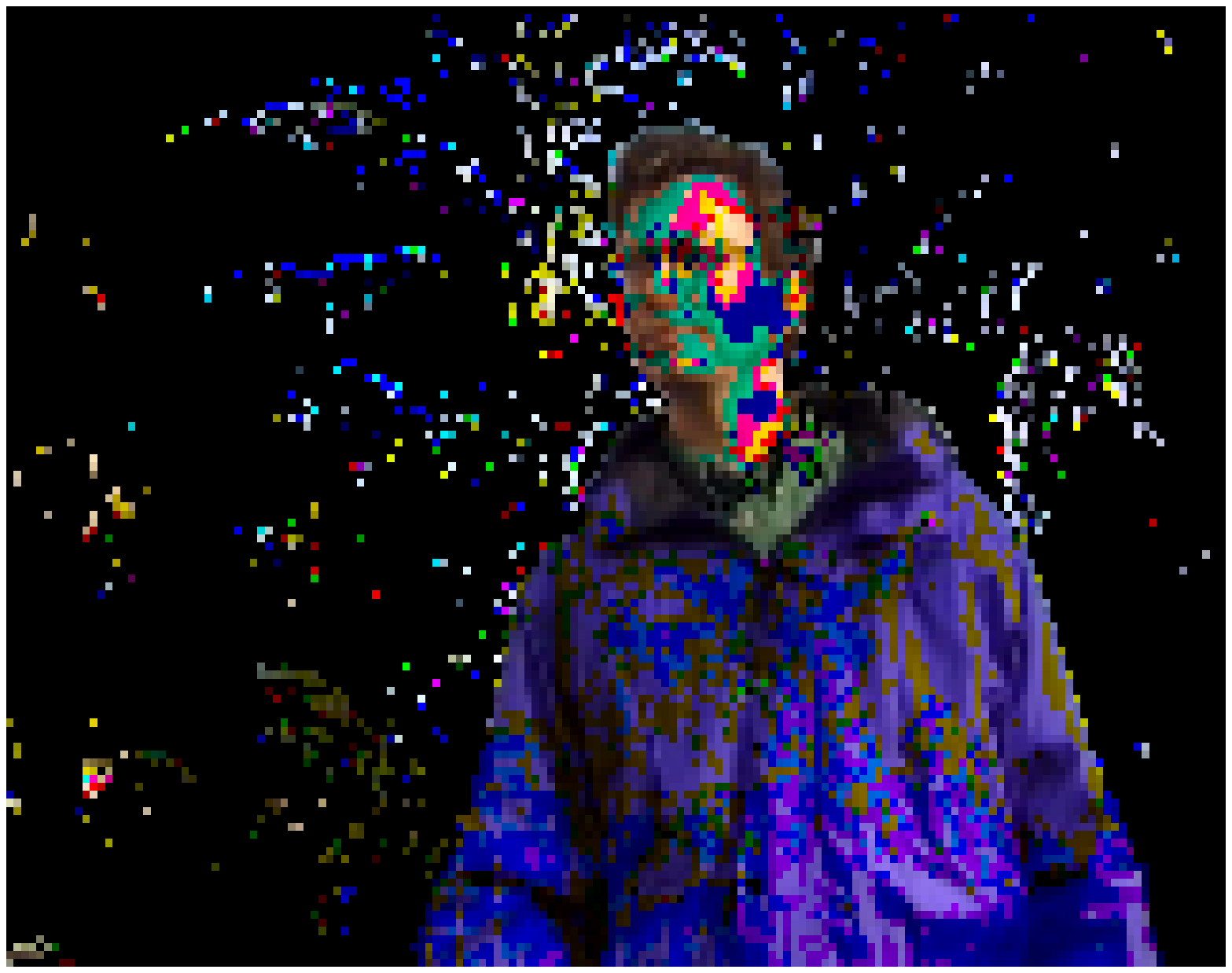}& \EcartTabFig
  \includegraphics[width=0.19\textwidth]{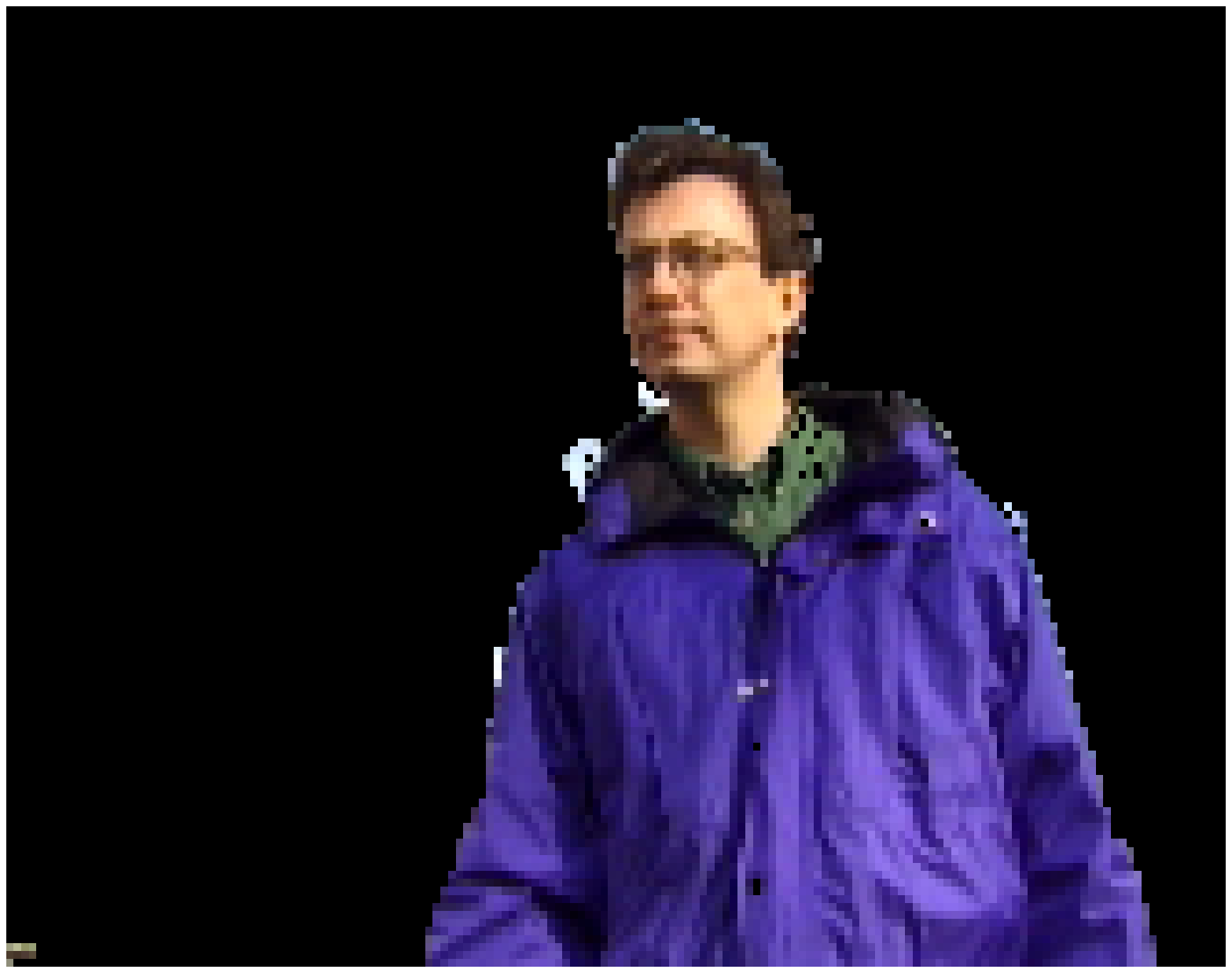} & \EcartTabFig
  \includegraphics[width=0.19\textwidth]{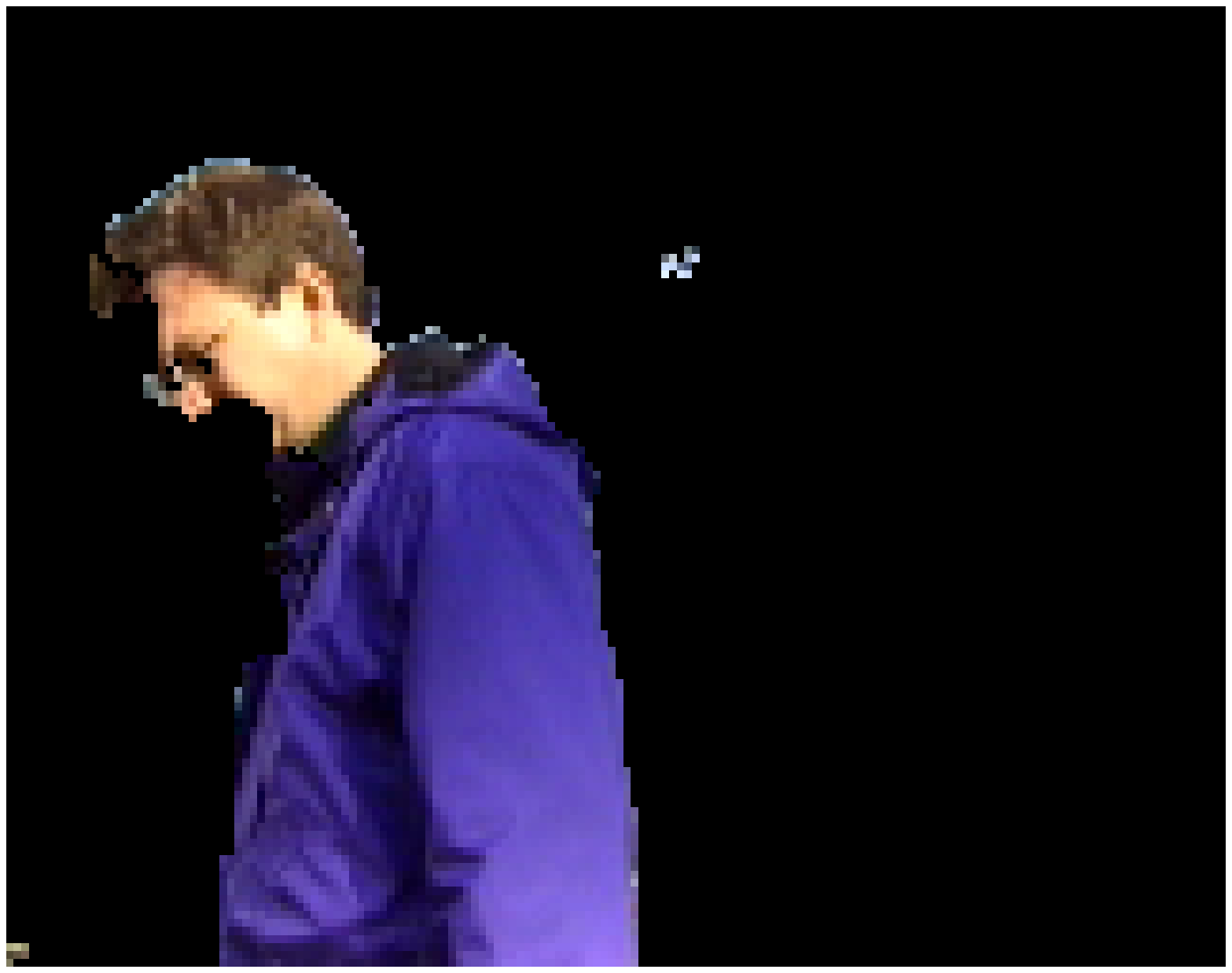} \\
  \includegraphics[width=0.19\textwidth]{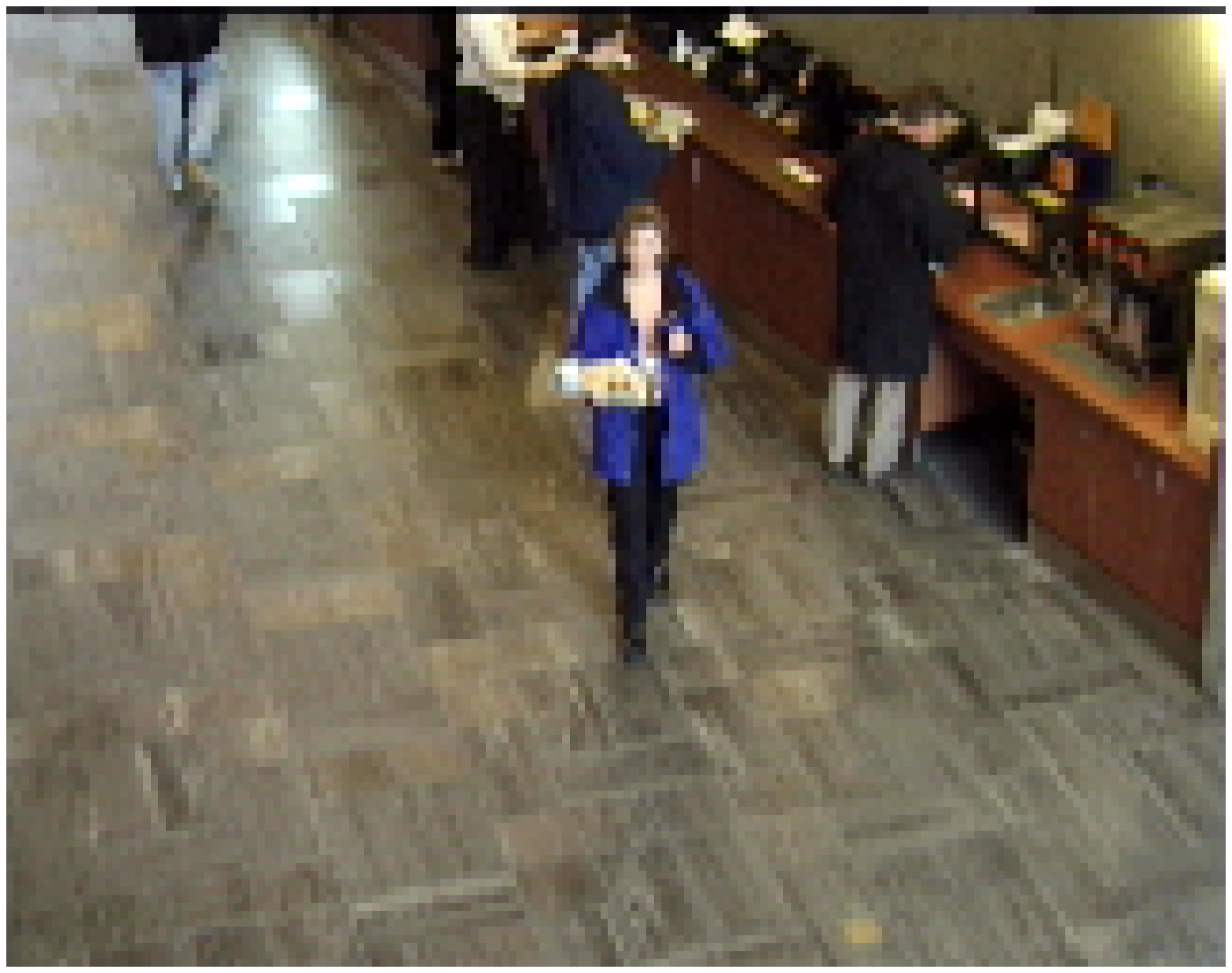} & \EcartTabFig
  \includegraphics[width=0.19\textwidth]{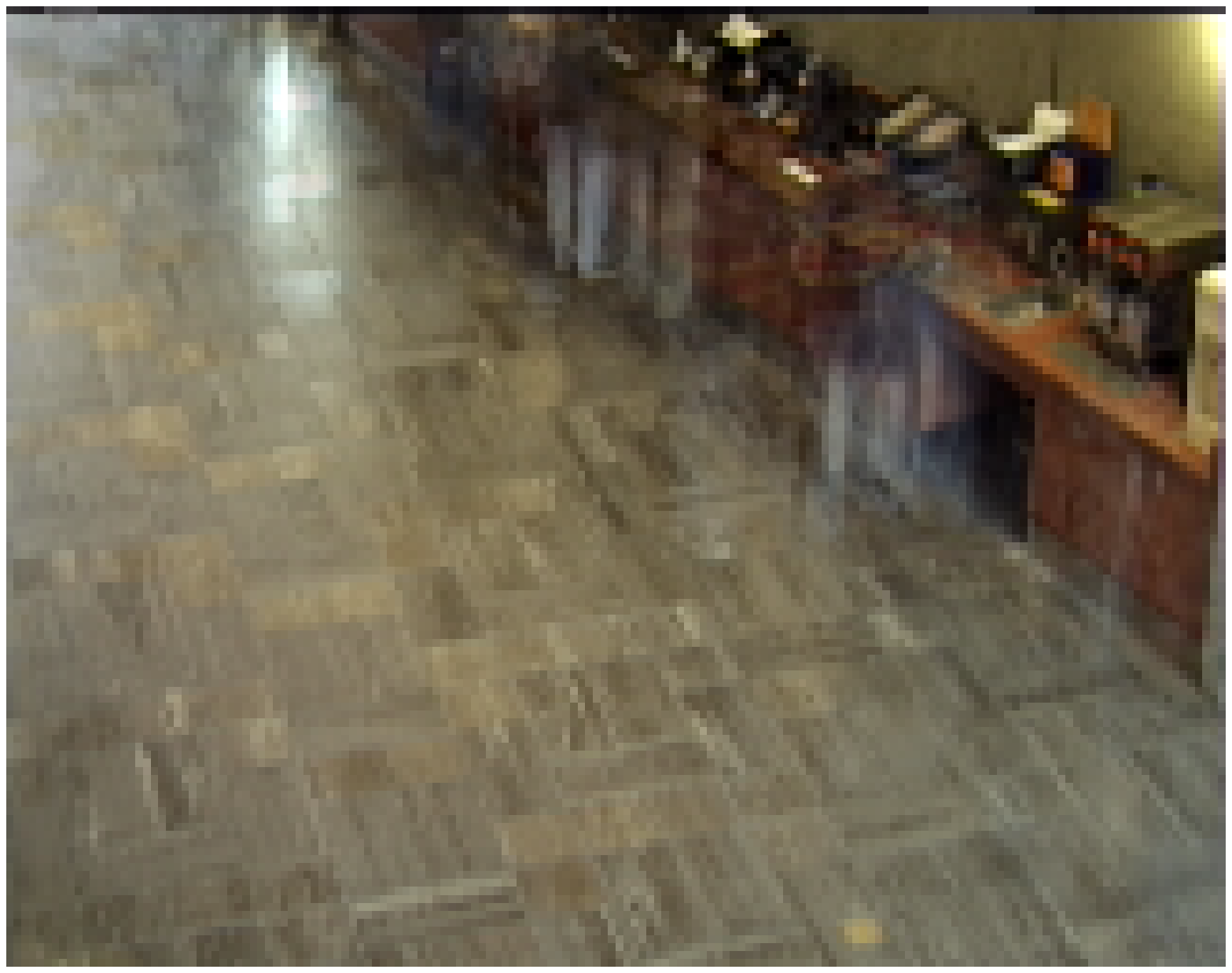} & \EcartTabFig
  \includegraphics[width=0.19\textwidth]{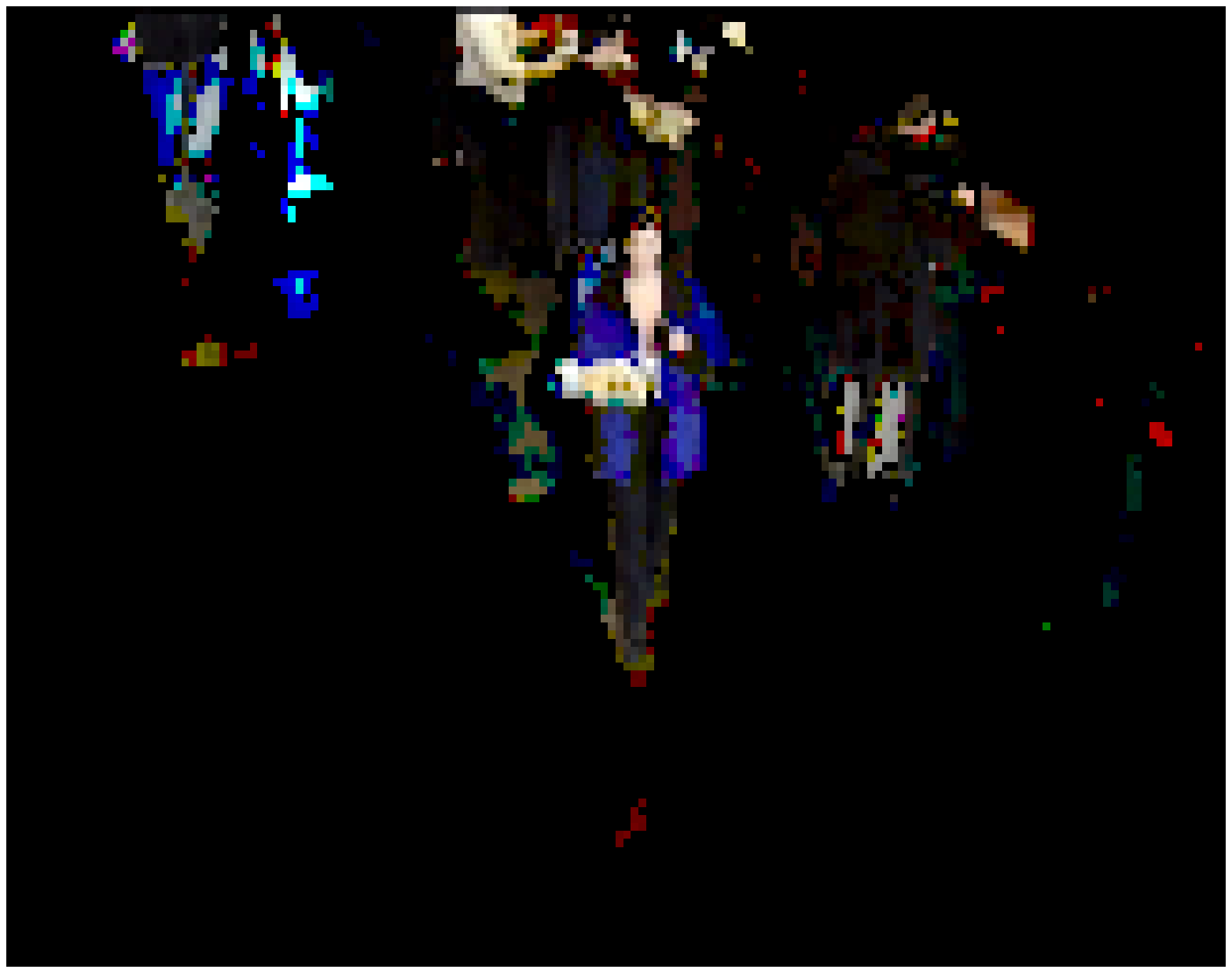} &\EcartTabFig
  \includegraphics[width=0.19\textwidth]{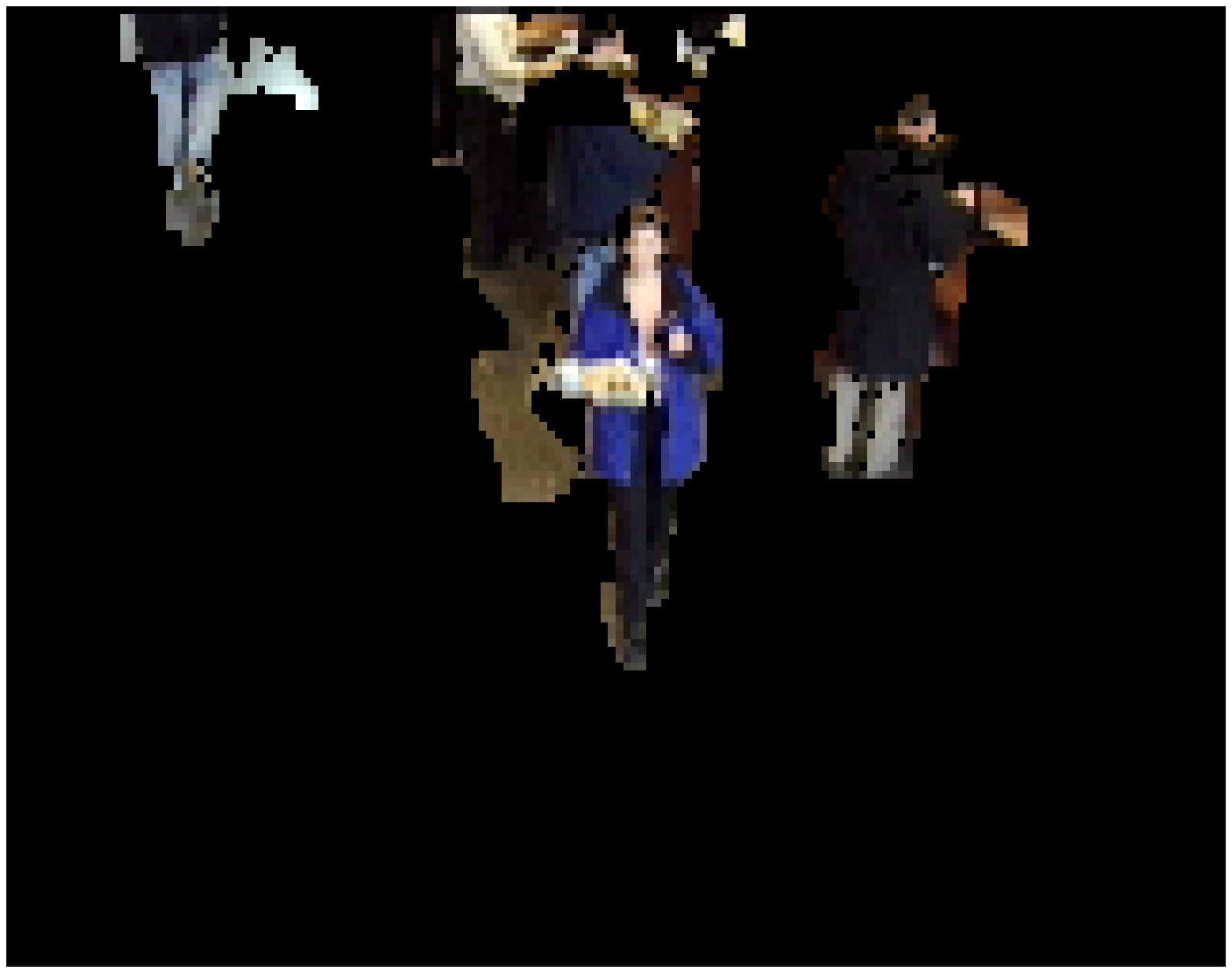} &\EcartTabFig
  \includegraphics[width=0.19\textwidth]{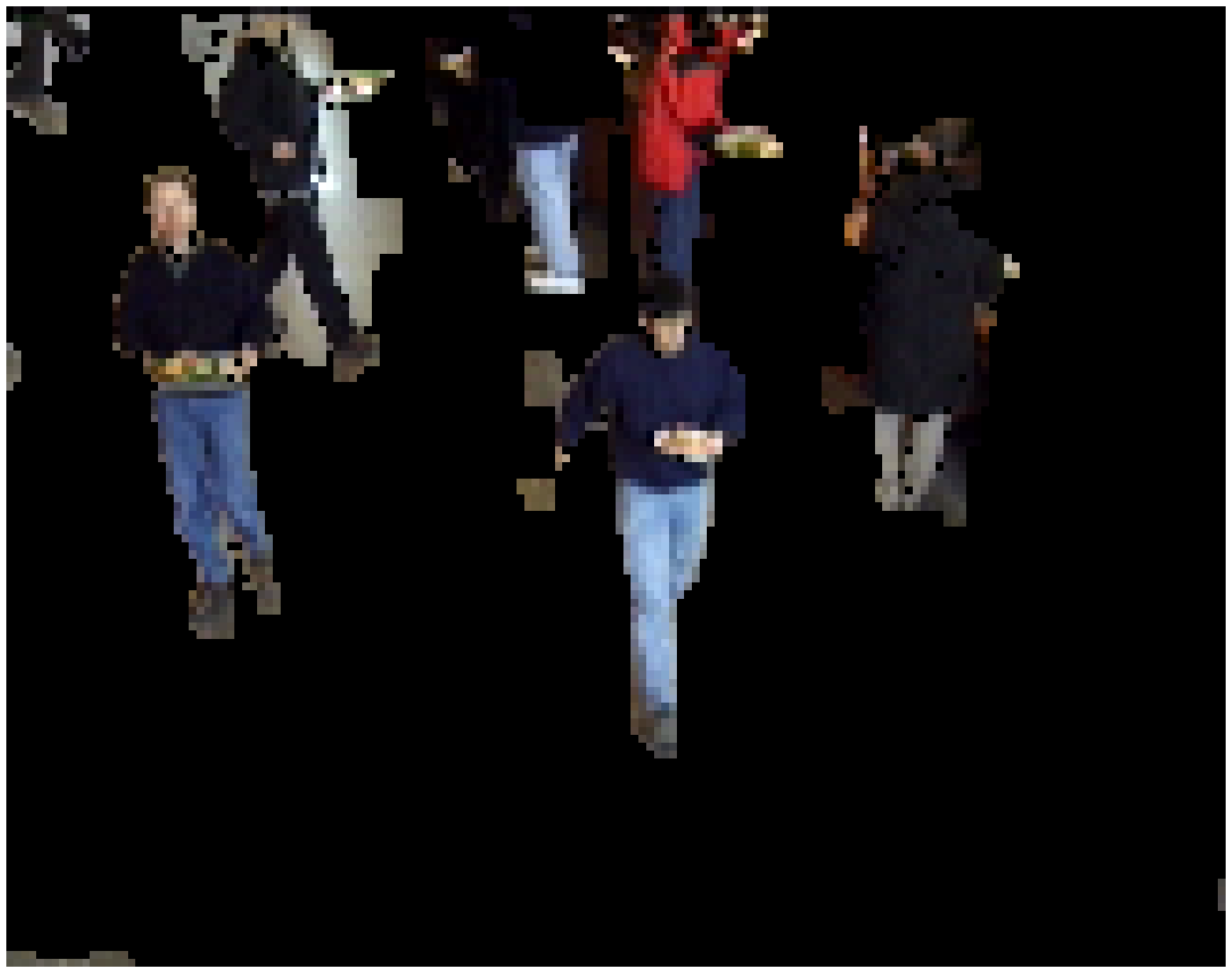}
 \end{tabular}
 \end{center}
\caption{The original image $\y$ (column 1), the background (i.e., $\X\w$) reconstructed by our method (column 2), 
and the foreground (i.e., the sparsity pattern of $\e$ as a mask on the original image) detected with $\ell_1$ (column 3) and with $\ell_1+\Omega$ (column 4).
The rightmost column is another foreground found with $\Omega$, on a different image,
with the same values of $\lambda_1,\lambda_2$ as for the previous image.
For the top left image, 
the percentage of pixels matching the ground truth is 98.8\% with $\Omega$, 87.0\% without.
As for the bottom left image, the result is 93.8\% with $\Omega$, 90.4\% without (best seen in color).}\label{fig:background_sub}
\end{figure}
\subsection{Multi-Task Learning of Hierarchical Structures}
In \cite{jenatton3}, Jenatton et al.~have recently proposed to use a
hierarchical structured norm to learn dictionaries of natural image patches.
Following their work, 
we seek to represent $n$ signals $\{\y^1,\dots,\y^n\}$ of dimension $m$ as sparse linear combinations of elements
from a dictionary $\X=[\x^1,\ldots,\x^p]$ in $\Real^{m \times p}$.
This can be expressed for all $i$ in $\IntSet{n}$ as $\y^i \approx \X \w^i$, for some sparse vector $\w^i$ in $\R{p}$.
In \cite{jenatton3}, the dictionary elements are embedded in a
\textit{predefined} tree $\mathcal{T}$, 
via a particular instance of the
structured norm $\Omega$, which we refer to it as~$\Omega_{\text{tree}}$,
 and call $\GG$ the underlying set of groups.
In this case, each signal~$\y^i$ admits a sparse decomposition in the
form of a subtree of dictionary elements. 

Inspired by ideas from multi-task learning \cite{obozinski}, we propose to
learn the tree structure $\mathcal{T}$ by pruning irrelevant parts of a larger
initial tree $\mathcal{T}_0$.  We achieve this by using an additional
regularization term~$\Omega_{\text{joint}}$ across the different decompositions,
so that subtrees of $\mathcal{T}_0$ will \textit{simultaneously} be removed for all signals $\y^i$.
In other words, the approach of \cite{jenatton3} is extended by the following formulation:
\begin{equation}
    \min_{\X,\W}
    \frac{1}{n}\sum_{i=1}^n\!\Big[\frac{1}{2} \|\y^i-\X\w^i\|_2^2 + \lambda_1 \Omega_{\text{tree}}(\w^i)\Big]\!+\!\lambda_2\Omega_{\text{joint}}(\W),\ \text{s.t.}\
    \|\x^j\|_2\leq 1, ~\text{for all}~ j ~\text{in}~ \IntSet{p}, \label{eq:dict_learning}
\end{equation}
where $\W \defin [\w^1,\ldots,\w^n]$ is the matrix of decomposition
coefficients in $\Real^{p \times n}$. The new regularization term operates
on the rows of $\W$ and is defined as $\Omega_{\text{joint}}(\W) \defin
\sum_{g\in\G}\max_{i\in\IntSet{n}}|\w_g^i|$.\footnote{The simplified case where 
$\Omega_{\text{tree}}$ and $\Omega_{\text{joint}}$ are the $\ell_1$-
and mixed $\ell_1/\ell_2$-norms~\cite{yuan} corresponds to~\cite{sprechmann}.}
The overall penalty on $\W$, which results from the combination of $\Omega_{\text{tree}}$ and
$\Omega_{\text{joint}}$, is itself an instance of~$\Omega$ with general overlapping groups, as defined in Eq~(\ref{eq:def_omega}).

To address problem~(\ref{eq:dict_learning}), we use the same optimization
scheme as \cite{jenatton3}, i.e., alternating between~$\X$ and~$\W$, fixing one
variable while optimizing with respect to the other. 
The task we consider is the denoising of natural image patches, with the same dataset and protocol as~\cite{jenatton3}.
We study whether learning the hierarchy of the dictionary elements improves the denoising performance, 
compared to standard sparse coding (i.e., when $\Omega_{\text{tree}}$ is the $\ell_1$-norm and $\lambda_2=0$) 
and the hierarchical dictionary learning of \cite{jenatton3} based on predefined trees (i.e., $\lambda_2=0$).
The dimensions of the training set --- $50\,000$ patches of size $8\!\times\!8$ for dictionaries with up to $p=400$ elements --- 
impose to handle extremely large graphs, with $|E|\approx|V|\approx 4.10^7$. 
Since problem~(\ref{eq:dict_learning}) is too large to be solved exactly sufficiently many times to select the regularization parameters $(\lambda_1,\lambda_2)$ rigorously,
we use the following heuristics:
we optimize mostly with the currently pruned tree held fixed (i.e., $\lambda_2=0$), 
and only prune the tree (i.e., $\lambda_2>0$) every few steps on a random subset of $10\,000$ patches.
We consider the same hierarchies as in~\cite{jenatton3}, involving between $30$ and $400$ dictionary elements.
The regularization parameter $\lambda_1$ is selected on the validation set of $25\,000$ patches, 
for both sparse coding (Flat) and hierarchical dictionary learning (Tree). 
Starting from the tree giving the best performance 
(in this case the largest one, see Figure~\ref{fig:tree}), 
we solve problem~(\ref{eq:dict_learning}) following our heuristics, for increasing values of $\lambda_2$.
As shown in Figure~\ref{fig:tree}, there is a regime where our approach performs significantly better than the two other compared methods.
The standard deviation of the noise is $0.2$ (the pixels have values in $[0,1]$); no significant improvements were observed for lower levels of noise.
\begin{figure}[hbtp]
   \centering
   \includegraphics[width=0.39\linewidth]{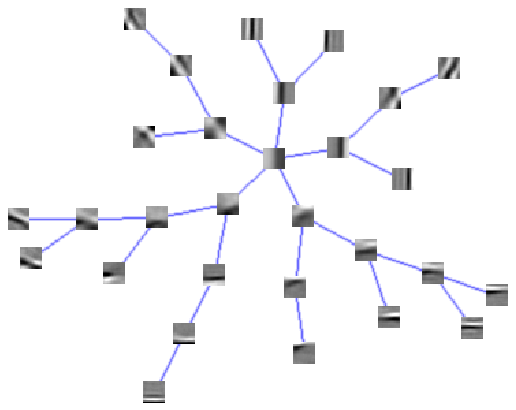}\hfill
   \includegraphics[width=0.5\linewidth]{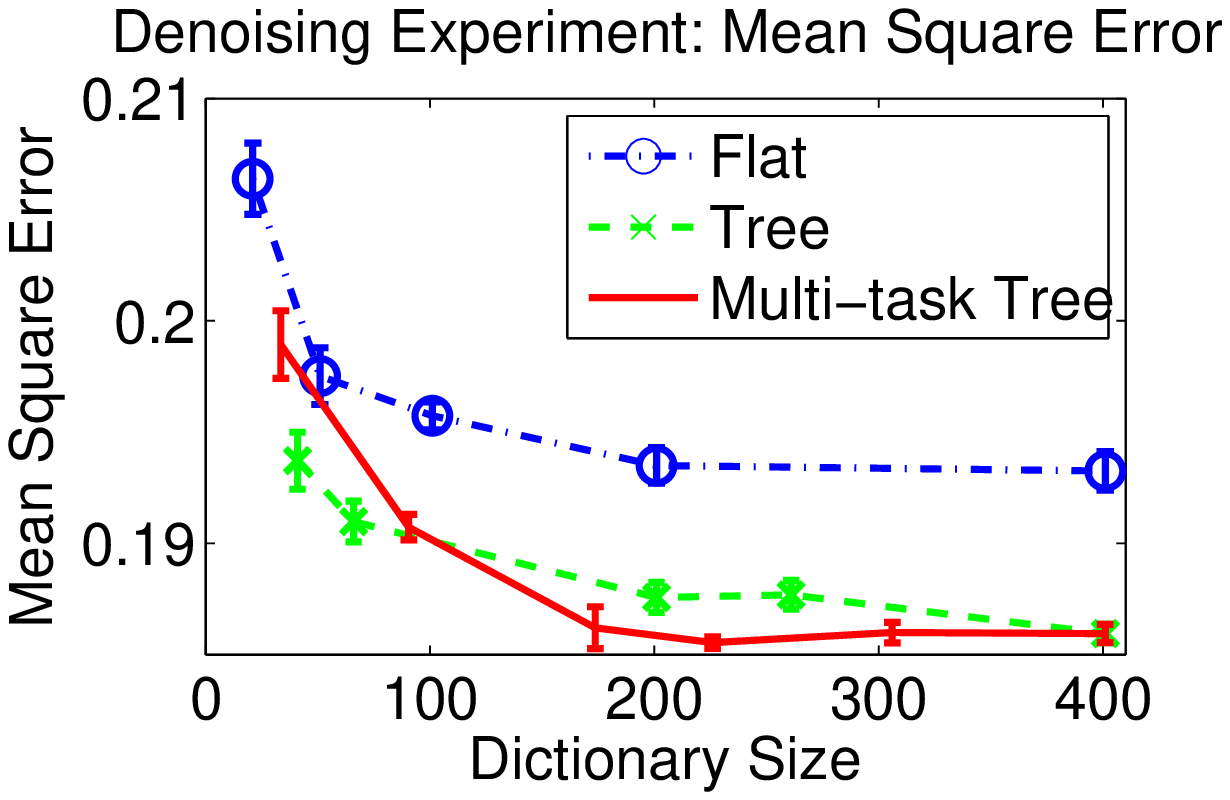} 
   \caption{Left: Hierarchy obtained by pruning a larger tree of $76$ elements. 
   Right: Mean square error versus dictionary size. 
   The error bars represent two standard deviations, based on three runs.} \label{fig:tree}
\end{figure}

 \section{Conclusion}
 We have presented a new optimization framework for solving sparse structured
problems involving sums of $\ell_\infty$-norms of any (overlapping) groups of
variables.  Interestingly, this sheds new light on connections between sparse
methods and the literature of network flow optimization.  In particular, the
proximal operator for the formulation we consider can be cast as a quadratic
min-cost flow problem, for which we propose an efficient and simple algorithm.
This allows the use of accelerated gradient methods.  Several experiments
demonstrate that our algorithm can be applied to a wide class of learning problems,
which have not been addressed before within sparse methods.

 \newpage
 \appendix
 \section{Equivalence to Canonical Graphs}\label{appendix:equivalent}
Formally, the notion of equivalence between graphs can be summarized by
the following lemma:
\begin{lemma}[Equivalence to canonical graphs.] \label{lemma:equivalent}~\\
Let $G=(V,E,s,t)$ be the canonical graph corresponding to a group structure $\GG$ with weights $(\eta_g)_{g \in \GG}$.
Let $G'=(V,E',s,t)$ be a graph sharing the same set of
vertices, source and sink as $G$, but with a different arc set $E'$.
We say that $G'$ is equivalent to $G$ if and only if the following conditions hold:
\begin{itemize}
\item Arcs of $E'$ outgoing from the source are the same as in $E$, with the same costs and capacities.
\item Arcs of $E'$ going to the sink are the same as in $E$, with the same costs and capacities. 
\item For every arc $(g,j)$ in $E$, with $(g,j)$ in $V_{gr} \times V_u$, there exists
a unique path in $E'$ from $g$ to $j$ with zero costs and infinite capacities on every
arc of the path.
\item Conversely, if there exists a path in $E'$ between a vertex $g$ in $V_{gr}$
and a vertex $j$ in $V_u$, then there exists an arc $(g,j)$ in $E$.
\end{itemize}
Then, the cost of the optimal min-cost flow on $G$ and $G'$ are the same.
Moreover, the values of the optimal flow on the arcs $(j,t)$, $j$ in $V_u$, are
the same on $G$ and $G'$.
\end{lemma}
\begin{proof}
We first notice that on both $G$ and $G'$, the cost of a flow on the graph only
depends on the flow on the arcs $(j,t)$, $j$ in $V_u$, which we have denoted by
$\xibbar$ in $E$.

We will prove that finding a feasible flow $\pi$ on $G$ with a cost $c(\pi)$ is
equivalent to finding a feasible flow $\pi'$ on $G'$ with the same cost
$c(\pi)=c(\pi')$.  We now use the concept of \emph{path flow}, which is a flow
vector in $G$ carrying the same positive value on every arc of a directed path
between two nodes of $G$.  It intuitively corresponds to sending a positive
amount of flow along a path of the graph.

According to the definition of graph equivalence introduced in the Lemma, it is
easy to show that there is a bijection between the arcs in $E$, and the paths
in $E'$ with positive capacities on every arc.  Given now a feasible flow $\pi$
in $G$, we build a feasible flow $\pi'$ on $G'$ which is a \textit{sum} of path
flows.  More precisely, for every arc $a$ in $E$, we consider its equivalent
path in $E'$, with a path flow carrying the same amount of flow as $a$.
Therefore, each arc $a'$ in $E'$ has a total amount of flow that is equal to
the sum of the flows carried by the path flows going over $a'$.  It is also
easy to show that this construction builds a flow on $G'$ (capacity and
conservation constraints are satisfied) and that this flow $\pi'$ has the same
cost as $\pi$, that is, $c(\pi)=c(\pi')$.

Conversely, given a flow $\pi'$ on $G'$, we use a classical path flow
decomposition~(see Proposition 1.1 in~\cite{bertsekas2}), saying that there
exists a decomposition of $\pi'$ as a sum of path flows in $E'$. Using the
bijection described above, we know that each path in the previous sums
corresponds to a unique arc in $E$. We now build a flow $\pi$ in $G$, by
associating to each path flow in the decomposition of $\pi'$, an arc in $E$
carrying the same amount of flow.  The flow of every other arc in $E$ is set to
zero.  It is also easy to show that this builds a valid flow in $G$ that has
the same cost as $\pi'$.
\end{proof}

\section{Convergence Analysis}\label{appendix:convergence}
We show in this section the correctness of Algorithm~\ref{algo:prox} for computing
the proximal operator, and of Algorithm~\ref{algo:dual_norm} for computing the dual
norm $\Omega^\star$.
\subsection{Computation of the Proximal Operator}
We now prove that our algorithm converges and that it finds the optimal solution of
the proximal problem. This requires that we introduce the optimality conditions for problem~(\ref{eq:dual_problem}) derived in~\cite{jenatton3}, since our convergence proof essentially checks that these conditions are satisfied upon termination of the algorithm.

\begin{lemma}[Optimality conditions of the problem~(\ref{eq:dual_problem}), \cite{jenatton3}]\label{lemma:opt}
The primal-dual variables $(\w,\xib)$ are respectively solutions of the primal~(\ref{eq:prox_problem})
and dual problems~(\ref{eq:dual_problem}) if and only if 
the dual variable $\xib$ is feasible for the problem~(\ref{eq:dual_problem}) and
\begin{displaymath}
\begin{split}
    & {\textstyle \w = \u -\sum_{g \in \G} \xib^g},  \\
  &  \forall g \in \G,~~ \left\{ \begin{array}{l}
       \w_g^\top \xib_g^g = \|\w_g\|_\infty \|\xib^g\|_1  ~~\text{and}~~ \|\xib^g\|_1=\lambda \eta_g,  \\
       \text{or}~~ \w_g = 0. 
    \end{array} \right. 
 \end{split}
 \end{displaymath}
\end{lemma}

Note that these optimality conditions provide an intuitive view of our
min-cost flow problem. Solving the min-cost flow problem is equivalent to
sending the maximum amount of flow in the graph under the capacity constraints,
while respecting the rule that \emph{the flow outgoing from a group $g$ should
always be directed to the variables $\u_j$ with maximum residual $\u_j-\sum_{g
\in \G}\xib^g_j$}.

Before proving the convergence and correctness of our algorithm, we also recall
classical properties of the min capacity cuts, which we
intensively use in the proofs of this paper.  The procedure
\texttt{computeFlow} of our algorithm finds a minimum $(s,t)$-cut of a graph
$G=(V,E,s,t)$, dividing the set $V$ into two disjoint parts $V^+$ and $V^-$.
$V^+$ is by construction the sets of nodes in $V$ such that there exists a
non-saturating path from $s$ to $V$, while all the paths from $s$ to $V^-$ are
saturated. Conversely, arcs from $V^+$ to $t$ are all saturated, whereas there
can be non-saturated arcs from $V^-$ to $t$. Moreover, the following properties
hold
\begin{itemize}
\item There is no arc going from $V^+$ to $V^-$. Otherwise the value of the cut would
be infinite. (Arcs inside $V$ have infinite capacity by construction of our graph). 
\item There is no flow going from $V^-$ to $V^+$ (see properties of the minimum $(s,t)$-cut~\cite{bertsekas2}).
\item The cut goes through all arcs going from $V^+$ to $t$, and all arcs going from $s$ to $V^-$.
\end{itemize}
All these properties are illustrated on Figure~\ref{fig:graph2}.
\begin{figure}[hbtp]
\tikzstyle{source}=[circle,thick,draw=blue!75,fill=blue!20,minimum size=8mm]
\tikzstyle{sink}=[circle,thick,draw=blue!75,fill=blue!20,minimum size=8mm]
\tikzstyle{group}=[place,thick,draw=red!75,fill=red!20, minimum size=8mm]
\tikzstyle{var}=[rectangle,thick,draw=black!75,fill=black!20,minimum size=6mm]
\tikzstyle{part}=[minimum size=6mm]
\tikzstyle{every label}=[red]
   \begin{center}
      \begin{tikzpicture}[node distance=1.7cm,>=stealth',bend angle=45,auto]
         \begin{scope}
            \node [source]   (s)                                    {$s$};
            \node [group]    (g)  [below of=s,xshift=-12mm]                      {$g$}
            edge  [pre] node[left] {$\xib^g_1 \!+ \!\xib^g_2 \! < \!\lambda \eta_g$} (s);
            \node [part] (p) [left of=g] {{\huge $V^+$}};
            \node [group]    (h)  [below of=s,xshift=12mm]                      {$h$}
            edge  [pre,very thick] node[right] {$\xib^h_2 \! = \!\lambda \eta_h$} (s);
            \node [part] (p2) [right of=h] {{\huge $V^-$}};
            \node [var] (u2) [below of=g,xshift=12mm]                    {$\u_2$}
            edge  [pre,dotted] node[above, right] {$0$} (h)
            edge  [pre] node[above, left] {$\xib^{g}_2$} (g);
            \node [var] (u1)  [left of=u2, node distance=2.5cm] {$\u_1$}
            edge  [pre] node[above, left] {$\xib^{g}_1$} (g);
            \node [var] (u3) [right of=u2, node distance=2.5cm] {$\u_3$}
            edge  [pre] node[above, right] {$\xib^{h}_3$} (h);
            \node [sink] (si) [below of=u2,node distance=2cm] {$t$}
            edge [pre,very thick] node[above,left,xshift=1mm] {$\xibbar_1 \!\! =\!\! \gammab_1$} (u1)
            edge [pre,very thick] node[above,left,xshift=1mm] {$\xibbar_2\!\! = \!\! \gammab_2$} (u2)
            edge [pre] node[above,right] {$\xibbar_3 \! < \! \gammab_3$} (u3);
            \draw [very thick] (1,0) .. controls (0,-0.75) .. (0,-1.5) .. controls (0,-2.5) .. (0.75,-3) .. controls (1.5,-3.5) and (1.5,-4) .. (0.5,-4) -- (-3,-4);
         \end{scope}
      \end{tikzpicture}
   \end{center}
   \caption{Cut computed by our algorithm. $V^+\! = \! V_u^+ \cup V_{gr}^+$, with $V_{gr}^+\! =\! \{g\}$, $V_{u}^+\! =\! \{ 1,2 \}$,
and $V^- \! =\! V_u^- \cup V_{gr}^-$, with $V_{gr}^-\! =\! \{h\}$, $V_{u}^- \! =\! \{ 3 \}$. Arcs going from $s$ to $V^-$ are saturated, as well as arcs going from $V^+$ to $t$. Saturated arcs are in bold. Arcs with zero flow are dotted.} \label{fig:graph2}
\end{figure}

Recall that we assume (cf. Section~\ref{subsec:graph})
that the scalars $\u_j$ are all non negative, and that we add non-negativity
constraints on~$\xib$.  With the optimality conditions of Lemma~\ref{lemma:opt}
in hand, we can show our first convergence result.
\begin{proposition}[Convergence of Algorithm \ref{algo:prox}] \label{prop:convergence}~\\
Algorithm~\ref{algo:prox} converges in a finite and polynomial number of operations.
\end{proposition}
\begin{proof}
Our algorithm splits recursively the graph into disjoints parts and processes
each part recursively.  The processing of one part requires an orthogonal
projection onto an $\ell_1$-ball and a max-flow algorithm, which can both be
computed in polynomial time.  To prove that the procedure converges, it is
sufficient to show that when the procedure \texttt{computeFlow} is called
for a graph $(V,E,s,t)$ and computes a cut $(V^+, V^-)$, then the components
$V^+$ and $V^-$ are both non-empty.

Suppose for instance that $V^- \!\!= \emptyset$.  In this case, the capacity of
the min-cut is equal to $\sum_{j \in V_u} \gammab_j$, and the value of the
max-flow is $\sum_{j \in V_u} \xibbar_j$. Using the classical max-flow/min-cut
theorem~\cite{ford}, we have equality between these two terms.  Since, by
definition of both $\gammab$ and $\xibbar$, we have for all~$j$ in~$V_u$,
$\xibbar_j \leq \gammab_j$, we obtain a contradiction with the existence of $j$
in $V_u$ such that $\xibbar_j \neq \gammab_j$. 

Conversely, suppose now that $V^+ \!\!= \emptyset$.
Then, the value of the max-flow is still $\sum_{j \in V_u} \xibbar_j$, and the value
of the min-cut is $\lambda\sum_{g \in V_{gr}} \eta_g$. Using again the
max-flow/min-cut theorem, we have that $\sum_{j \in V_u} \xibbar_j = \lambda\sum_{g
\in V_{gr}} \eta_g$.  Moreover, by definition of $\gammab$, we also have
$\sum_{j \in V_u} \xibbar_j \leq \sum_{j \in V_u} \gammab_j \leq \lambda\sum_{g
\in V_{gr}} \eta_g$, leading to a contradiction with the existence of $j$ in
$V_u$ such that $\xibbar_j \neq \gammab_j$. This proof holds for any graph that
is equivalent to the canonical one.
\end{proof}
After proving the convergence, we prove that the algorithm is correct with the next proposition.

\begin{proposition}[Correctness of Algorithm \ref{algo:prox}] ~\\
 Algorithm~\ref{algo:prox} solves the proximal problem of Eq.~(\ref{eq:prox_problem}).
\end{proposition}
\begin{proof}
For a group structure $\GG$, we first prove the correctness of our algorithm if the graph used is its associated canonical graph that we denote $G_0=(V_0,E_0,s,t)$.
We proceed by induction on the number of nodes of the graph.
The induction hypothesis $\H(k)$ is the following:~\vspace*{0.2cm}\\
\textit{For all canonical graphs $G=(V = V_u \cup V_{gr},E,s,t)$ associated with a group structure $\GG_V$ with weights $(\eta_g)_{g \in \GG_V}$ such that $|V|\leq k$, \texttt{computeFlow}$(V,E)$
solves the following optimization problem:}
\begin{equation}
\min_{(\xib_j^g)_{j \in V_u,g \in V_{gr}}} \sum_{j \in V_u} \frac{1}{2} (\u_j - \sum_{g \in V_{gr}}
\xib_j^g)^2 \st \forall g \in V_{gr},~ \sum_{j \in V_u} \xib_j^g \leq \lambda \eta_g
~~\text{and}~~ \xib_j^g=0,~\forall j \notin g. \label{eq:dual2}
\end{equation}
Since $\GG_{V_0}=\GG$, it is sufficient to show that $\H(|V_0|)$ to prove the proposition.

We initialize the induction by $\H(2)$, corresponding to the simplest canonical graph, for which $|V_{gr}|=|V_u|=1$).
Simple algebra shows that $\H(2)$ is indeed correct.

We now suppose that $\H(k')$ is true for all $k' < k$ and consider a graph
$G=(V,E,s,t)$, $|V|=k$.  The first step of the algorithm computes the variable
$(\gammab_j)_{j\in V_u}$ by a projection on the $\ell_1$-ball. This is itself
an instance of the dual formulation of Eq.~(\ref{eq:dual_problem}) in a
simple case, with one group containing all variables.  We can therefore use
Lemma~\ref{lemma:opt} to characterize the optimality of
$(\gammab_j)_{j\in V_u}$, which yields 
\begin{equation}
    \!\!\left\{ \begin{array}{l}
       \sum_{j\in V_u} (\u_j-\gammab_j)\gammab_j = \big(\max_{j \in V_u} |\u_j-\gammab_j|\big)
  \sum_{j\in V_u} \gammab_j ~\text{and}~ \sum_{j\in V_u} \gammab_j=\lambda \sum_{g \in V_{gr}} \eta_g,\!\!\!\!
\\
       \text{or}~~ \u_j - \gammab_j = 0,~\forall j \in V_u. \label{eq:opt_gamma}
    \end{array} \right. 
\end{equation}
The algorithm then computes a max-flow, using the scalars $\gammab_j$ as
capacities, and we now have two possible situations:\\
\begin{enumerate}
\item If $\xibbar_j = \gammab_j$ for
all $j$ in $V_u$, the algorithm stops; we write $\w_j = \u_j-\xibbar_j$ for $j$ in $V_u$, and using Eq.~(\ref{eq:opt_gamma}), we obtain
\begin{equation}
    \left\{ \begin{array}{l}
       \sum_{j\in V_u} \w_j\xibbar_j = (\max_{j \in V_u} |\w_j|)
  \sum_{j\in V_u} \xibbar_j ~~\text{and}~~ \sum_{j\in V_u} \xibbar_j=\lambda \sum_{g \in V_{gr}} \eta_g,
\\
       \text{or}~~ \w_j = 0,~\forall j \in V_u.
    \end{array} \right. 
\end{equation}
We can rewrite the condition above as 
$$
\sum_{g\in V_{gr}}\sum_{j\in g}\! \w_j\xib_j^g = \sum_{g\in V_{gr}}\! ( \max_{j \in V_u} |\w_j| )\!\!\! \sum_{j\in V_u}\xib_j^g.
$$
Since all the quantities in the previous sum are positive, this can only hold if for all $g\in V_{gr}$,
$$
\sum_{j\in V_u}\w_j\xib_j^g = ( \max_{j \in V_u} |\w_j| )\! \sum_{j\in V_u}\xib_j^g.
$$
Moreover, by definition of the max flow and the optimality conditions, we have
$$
\forall g\in V_{gr},\ \sum_{j\in V_u}\! \xib^g_j \leq \lambda \eta_g,\ \text{and}\ \sum_{j\in V_u} \xibbar_j=\lambda \sum_{g \in V_{gr}} \eta_g,
$$
which leads to
$$
\forall g\in V_{gr}, \sum_{j\in V_u}\! \xib^g_j  = \lambda \eta_g.
$$
By Lemma~\ref{lemma:opt}, we have shown that the problem~(\ref{eq:dual2}) is solved.

\item Let us now consider the case where there exists $j$ in $V_u$ such that
$\xibbar_j \neq \gammab_j$. The algorithm splits the vertex set $V$
into two parts $V^+$ and $V^-$, which we have proven to be non-empty in the proof
of Proposition~\ref{prop:convergence}. The next step of the algorithm removes all edges between $V^+$ and $V^-$ (see Figure~\ref{fig:graph2}).
Processing $(V^+,E^+)$ and $(V^-,E^-)$ independently, it
updates the value of the flow matrix $\xib^g_j,\ j\in V_u,\ g\in V_{gr}$, 
and the corresponding flow vector~$\xibbar_j,\ j\in V_u$.
As for $V$, we denote by $V^+_u \defin V^+ \cap V_u$, $V^-_u \defin V^- \cap V_u$ and
$V^+_{gr} \defin V^+ \cap V_{gr}$, $V^-_{gr} \defin V^- \cap V_{gr}$.

Then, we notice that $(V^+,E^+,s,t)$ and $(V^-,E^-,s,t)$ are respective canonical graphs for the group
structures $\GG_{V^+} \defin \{ g \cap V_u^+ \mid g \in V_{gr} \}$, and $\GG_{V^-} \defin \{ g \cap V_u^- \mid g \in V_{gr} \}$.

Writing $\w_j = \u_j-\xibbar_j$ for $j$ in $V_u$, and using the induction hypotheses
$\H(|V^+|)$ and $\H(|V^-|)$, we now have the following optimality conditions deriving
from Lemma~\ref{lemma:opt} applied on Eq.~(\ref{eq:dual2}) respectively for the graphs
$(V^+,E^+)$ and $(V^-,E^-)$:
\begin{equation}
 \forall g \in V_{gr}^+,  g' \defin g \cap V_u^+,~
    \left\{ \begin{array}{l}
       \w_{g'}^\top \xib_{g'}^g = \|\w_{g'}\|_\infty 
 \sum_{j\in g'}\! \xib^g_j ~~\text{and}~~ \sum_{j\in g'}\! \xib^g_j=\lambda \eta_g,
\\
       \text{or}~~ \w_{g'} = 0, \\
    \end{array} \right. \label{eq:hplus}
\end{equation}
and
\begin{equation}
\forall g \in V_{gr}^-, g' \defin g \cap V_u^-,  
    \left\{ \begin{array}{l}
       \w_{g'}^\top \xib_{g'}^g = \|\w_{g'}\|_\infty 
\sum_{j\in g'}\! \xib^g_j  ~~\text{and}~~ \sum_{j\in g'}\! \xib^g_j=\lambda \eta_g,
\\
       \text{or}~~ \w_{g'} = 0. \\
    \end{array} \right. 
\label{eq:hmoins}
\end{equation}
We will now combine Eq.~(\ref{eq:hplus}) and Eq.~(\ref{eq:hmoins}) into optimality conditions
for Eq.~(\ref{eq:dual2}).  We first notice that $g \cap V_u^+ = g$ since there
are no arcs between $V^+$ and $V^-$ in $E$ (see the properties of the cuts
discussed before this proposition).
It is therefore possible to replace $g'$ by $g$ in Eq.~(\ref{eq:hplus}). 
We will show that it is possible to do the same in Eq.~(\ref{eq:hmoins}), so that combining these two equations yield the optimality conditions of Eq.~(\ref{eq:dual2}).

More precisely, we will show that for all $g \in V_{gr}^-$ and $j \in
g \cap V_u^+$, $|\w_j| \leq \max_{l \in g \cap V_u^-} |\w_l|$, in which case $g'$ can be replaced by $g$ in Eq.~(\ref{eq:hmoins}).
This result is relatively intuitive: $(s,V^+)$ and $(V^-,t)$ being an $(s,t)$-cut,
all arcs between $s$ and $V^-$ are saturated, while there are unsaturated arcs
between $s$ and $V^+$; one therefore expects the residuals $\u_j -\xibbar_j$
to decrease on the $V^+$ side, while increasing on the $V^-$ side.
The proof is nonetheless a bit technical.

Let us show first that for all $g$ in $V_{gr}^+$, $\NormInf{\w_g} \leq \max_{j\in V_u}|\u_j-\gammab_j|$.
We split the set $V^+$ into disjoint parts:
\begin{displaymath}
\begin{split}
V_{gr}^{++} &\defin \{ g \in V_{gr}^+ \st \NormInf{\w_g} \leq \max_{j\in V_u}|\u_j-\gammab_j| \}, \\
V_{u}^{++} &\defin \{ j \in V_{u}^+ \st \exists g \in V_{gr}^{++},~ j \in g \}, \\
V_{gr}^{+-} &\defin V_{gr}^+ \setminus V_{gr}^{++} = \{ g \in V_{gr}^+ \st \NormInf{\w_g} > \max_{j\in V_u}|\u_j-\gammab_j| \}, \\
V_{u}^{+-} &\defin V_{u}^+ \setminus V_{u}^{++}. \\
\end{split}
\end{displaymath}
As previously, we denote $V^{+-}\! \defin V_u^{+-}\! \cup V_{gr}^{+-}$ and $V^{++} \defin\!\!
V_u^{++} \cup V_{gr}^{++}$.  We want to show that $V_{gr}^{+-}$ is necessarily empty.
We reason by contradiction and assume that $V_{gr}^{+-} \neq \varnothing$.

According to the definition of the different sets above, we observe that no
arcs are going from $V^{++}$ to $V^{+-}$, that is, for all $g$ in
$V_{gr}^{++}$, $g \cap V_{u}^{+-}=\varnothing$.  We observe as well that the
flow from $V_{gr}^{+-}$ to $V_u^{++}$ is the null flow, because optimality
conditions (\ref{eq:hplus}) imply that for a group $g$ only nodes $j \in g$
such that $\w_j=\|\w_g\|_\infty$ receive some flow, which excludes nodes in
$V_u^{++}$ provided $V_{gr}^{+-} \neq \varnothing$;
Combining this fact and the inequality 
$
\sum_{g \in V_{gr}^{+}} \lambda \eta_g
\geq \sum_{j \in V_u^{+}} \gammab_j
$ (which is a direct consequence of the minimum $(s,t)$-cut),
we have as well 
$$
\sum_{g \in V_{gr}^{+-}} \lambda \eta_g \geq \sum_{j \in V_u^{+-}} \gammab_j.
$$

Let $j \in V_u^{+-}$, if $\xibbar_j \neq 0$ then for some $g \in V_{gr}^{+-}$ such that $j$ receives
some flow from $g$, which from the optimality conditions~(\ref{eq:hplus}) implies $\w_j=\|\w_g\|_\infty$;
by definition of $V_{gr}^{+-}$, $\|\w_g\|_\infty > \u_j-\gammab_j$. But since at the optimum, $\w_j=\u_j-\xibbar_j$,
this implies that $\xibbar_j < \gammab_j$, and in turn that $\sum_{j\in V_u^{+-}} \xibbar_j=\lambda \sum_{g \in V_{gr}^{+-}} \eta_g$.
Finally, $$\lambda \sum_{g \in V_{gr}^{+-}} \eta_g=\sum_{j\in V_u^{+-},\, \xibbar_j \neq 0} \xibbar_j < \sum_{j\in V_u^{+-}} \gammab_j$$ and this is a contradiction.
%}

We now have that for all $g$ in $V_{gr}^+$, $\NormInf{\w_g} \leq \max_{j\in V_u}|\u_j-\gammab_j| $.
The proof showing that for all $g$ in $V_{gr}^-$, 
$
\NormInf{\w_g} \geq \max_{j\in V_u}|\u_j-\gammab_j|,
$ 
uses the same kind of decomposition for $V^-$, and follows along similar arguments. We will therefore not detail it.

To recap, we have shown that for all $g \in V_{gr}^-$ and $j \in
g \cap V_u^+$, $|\w_j| \leq \max_{l \in g \cap V_u^-} |\w_l|$.
Since there is no flow from $V^-$ to $V^+$, 
i.e., $\xib_j^g=0$ for $g$ in $V_{gr}^-$ and $j$ in $V_{u}^+$,
we can now replace the definition of $g'$ in Eq.~(\ref{eq:hmoins}) by $g'
\defin g \cap V_u$, the combination of Eq.~(\ref{eq:hplus}) and
Eq.~(\ref{eq:hmoins}) gives us optimality conditions for Eq.~(\ref{eq:dual2}).
\end{enumerate}

The proposition being proved for the canonical graph, we extend it now
for an equivalent graph in the sense of Lemma~\ref{lemma:equivalent}.
First, we observe that the algorithm gives the same values of $\gammab$
for two equivalent graphs. Then, it is easy to see that the value $\xibbar$
given by the max-flow, and the chosen $(s,t)$-cut is the same, which is
enough to conclude that the algorithm performs the same steps for
two equivalent graphs.
\end{proof}

\subsection{Computation of the Dual Norm $\Omega^\star$}
Similarly to the proximal operator, the computation of dual norm $\Omega^*$ can itself shown to solve another network flow problem, based on the following variational formulation, which extends a previous result from \cite{jenatton}:
%===============================================================================================================
\begin{lemma}[Dual formulation of the dual-norm $\Omega^\star$.] ~\\ Let $\kappab \in \R{p}$. We have
\begin{equation}
\Omega^*(\kappab) = \!\!\! \min_{\xib\in\RR{p}{|\G|},\tau} \!\!\! \tau \quad \text{s.t.}\quad \sum_{g\in\G}\xib^g=\kappab,\ \text{and}\ \forall g\in\G,\ \|\xib^g\|_1 \leq \tau\eta_g ~~\text{with}~~\
\xib_j^g=0\ \text{if}\ j\notin g. \label{eq:dual_norm}
\end{equation}
\end{lemma}

\begin{proof} 
By definition of $\Omega^*(\kappab)$, we have 
$$\Omega^*(\kappab)\defin\max_{\Omega(\z)\leq 1}\z^\top\kappab.$$
By introducing the primal variables $(\alpha_g)_{g\in\GG} \in \R{|\GG|}$, we can rewrite the previous maximization problem as
$$
 \Omega^*(\kappab) = \max_{ \sum_{g\in\GG} \!\eta_g\alpha_g \leq 1 } \kappab^\top \z,\quad \mbox{ s.t. } \quad \forall\ g\in\GG,\ \|\z_g\|_\infty \leq \alpha_g,
$$
with the additional $|\GG|$ conic constraints $ \|\z_g\|_\infty \leq \alpha_g$. 
This primal problem is convex and satisfies Slater's conditions for generalized conic inequalities, which implies that strong duality holds \cite{boyd}.
We now consider the Lagrangian $\mathcal{L}$ defined as
$$
\mathcal{L}(\z,\alpha_g,\tau, \gamma_g,\xib) 
= 
\kappab^\top \z + \tau (1 \! - \!\! \sum_{g\in\GG} \! \eta_g \alpha_g ) + 
\sum_{g\in\GG} \binom{\alpha_g}{\z_g}^\top \! \binom{\gamma_g}{\xib^g_g}, 
$$ 
with the dual variables 
$\{\tau, (\gamma_g)_{g\in\GG},\xib\} \in \R{}_+ \! \times \! \R{|\GG|} \! \times \!\RR{p}{|\GG|}$
such that for all $g\in\GG$, 
$\xib^g_j = 0 \, \mbox{ if } \, j \notin g$
and 
$ \|\xib^g\|_1 \leq \gamma_g $.
The dual function is obtained by taking the derivatives of $\mathcal{L}$ with respect to the primal variables $\z$ and $(\alpha_g)_{g\in\GG}$ and equating them to zero,
which leads to 
\begin{eqnarray*}
	\forall j\in \{1,\dots,p\},&      \kappab_j + \!\! \sum_{ g \in \GG } \xib^g_j & =  0 \\
	\forall g \in \GG,& \tau \eta_g -  \gamma_g & = 0.
\end{eqnarray*}
After simplifying the Lagrangian and flipping the sign of $\xib$, the dual problem then reduces to
$$
	\min_{ \xib\in\RR{p}{|\G|},\tau }  \tau  \quad \mbox{ s.t. }
	\begin{cases}
	 \forall j \in \{1,\dots,p\}, \kappab_j = \!\! \sum_{ g \in \GG } \xib^g_j\,\mbox{ and }\, \xib^g_j = 0 \, \mbox{ if } \, j \notin g,\\
	 \forall g \in \GG, \|\xib^g\|_1 \leq \tau\eta_g,
	\end{cases}
$$
which is the desired result.
\end{proof}
We now prove that Algorithm~\ref{algo:dual_norm} is correct.
\begin{proposition}[Convergence and correctness of Algorithm \ref{algo:dual_norm}] ~\\
 Algorithm~\ref{algo:dual_norm} computes the value of the dual norm of Eq.~(\ref{eq:dual_norm}) in a finite and polynomial number of operations.
\end{proposition}
\begin{proof}
The convergence of the algorithm only requires to show that the
cardinality of $V$ in the different calls of the function
\texttt{computeFlow} strictly decreases. Similar arguments
to those used in the proof of Proposition~\ref{prop:convergence} can
show that each part of the cuts $(V^+,V^-)$ are
both non-empty. The algorithm thus requires a finite number of
calls to a max-flow algorithm and converges in a finite
and polynomial number of operations.

Let us now prove that the algorithm is correct for a canonical graph. We
proceed again by induction on the number of nodes of the graph.  
More precisely, we consider the
induction hypothesis $\H'(k)$
defined as:~\vspace*{0.2cm}\\
\textit{for all canonical graphs} $G=(V,E,s,t)$ associated with a group structure $\GG_V$ and \textit{such that} $|V| \leq k$, \texttt{dualNormAux}$(V=V_u \cup V_{gr},E)$
\textit{solves the following optimization problem:}
\begin{equation}
 \min_{\xib,\tau} \tau \quad \text{s.t.}\quad \forall j \in V_u, \kappab_j = \sum_{g\in V_{gr}}\xib_j^g,\ \text{and}\ \forall g\in V_{gr},\ \sum_{j \in V_u} \xib^g_j \leq \tau\eta_g ~~\text{with}~~\
\xib_j^g=0\ \text{if}\ j\notin g. \label{eq:Hprime}
\end{equation}
We first initialize the induction by $\H(2)$ (i.e., with the simplest canonical graph, such that $|V_{gr}|=|V_u|=1$).
Simple algebra shows that $\H(2)$ is indeed correct.

We next consider a canonical graph $G=(V,E,s,t)$ such that $|V|=k$, and suppose that $\H'(k-1)$ is true.
After the max-flow step, we have two possible cases to discuss:
\begin{enumerate}
\item If $\xibbar_j = \gammab_j$ for all $j$ in $V_u$, the algorithm stops.
We know that any scalar $\tau$ such that the constraints of
Eq.~(\ref{eq:Hprime}) are all satisfied necessarily verifies $\sum_{g \in
V_{gr}} \tau \eta_g \geq \sum_{j \in V_{u}}\kappab_j$. We have indeed that
$\sum_{g \in V_{gr}} \tau \eta_g$ is the value of an $(s,t)$-cut in the graph,
and $\sum_{j \in V_{u}}\kappab_j$ is the value of the max-flow, and the
inequality follows from the max-flow/min-cut theorem~\cite{ford}. This gives a
lower-bound on $\tau$. Since this bound is reached, $\tau$ is necessarily
optimal.
\item We now consider the case where there exists $j$ in $V_u$ such that
$\xibbar_j \neq \kappab_j$, meaning that for the given value of $\tau$,
the constraint set of Eq.~(\ref{eq:Hprime}) is not feasible for $\xib$,
and that the value of $\tau$ should necessarily increase.
The algorithm splits the vertex set $V$ into two non-empty parts $V^+$ and
$V^-$ and we remark that there are no arcs going from $V^+$ to $V^-$, and no
flow going from $V^-$ to $V^+$. 
 Since the arcs going from $s$ to $V^-$ are 
 saturated, we have that $\sum_{g \in V_{gr}^-} \tau \eta_g \leq \sum_{j
 \in V_{u}^-}\kappab_j$.
 Let us now consider $\tau^\star$ the solution of Eq.~(\ref{eq:Hprime}).
 Using the induction hypothesis $\H'(|V^-|)$, the algorithm computes a new value
 $\tau'$ that solves Eq.~(\ref{eq:Hprime}) when replacing $V$ by $V^-$ and
 this new value satisfies the following inequality $\sum_{g \in V_{gr}^-} \tau'
 \eta_g \geq \sum_{j \in V_{u}^-}\kappab_j$. The value of $\tau'$ has therefore
 increased and the updated flow $\xib$ now satisfies the constraints of
 Eq.~(\ref{eq:Hprime}) and therefore $\tau'\geq \tau^\star$.
 Since there are no arcs going from $V^+$ to $V^-$, $\tau^\star$ is feasible
for Eq.~(\ref{eq:Hprime}) when replacing $V$ by $V^-$ and we have that
$\tau^\star \geq  \tau'$ and then $\tau'=\tau^\star$.
\end{enumerate}
To prove that the result holds for any equivalent graph, similar arguments to those used in
the proof of Proposition~\ref{prop:convergence} can be exploited, 
showing that the algorithm
computes the same values of $\tau$ and same $(s,t)$-cuts at each step.
\end{proof}

\section{Algorithm FISTA with duality gap} \label{appendix:fista}

In this section, we describe in details the algorithm FISTA \cite{beck} when applied to solve problem~(\ref{eq:formulation}), 
with a duality gap as stopping criterion.

Without loss of generality, let us assume we are looking for models of the form $\X\w$, for some matrix $\X \in \RR{n}{p}$ (typically, linear models where $\X$ is the data matrix of $n$ observations). 
Thus, we can consider the following primal problem 
\begin{equation}
   \min_{\w \in \Real^p} f(\X\w) + \lambda \Omega(\w),\label{eq:formulation_with_X}
\end{equation}
in place of~(\ref{eq:formulation}). 
Based on Fenchel duality arguments \cite{borwein},
$$
 f(\X\w)+ \lambda \Omega(\w) + f^*(-\kappab),\ \text{for}\ \w\in\R{p},\kappab\in\R{n}\ \text{and}\
\Omega^*(\X^\top\kappab) \leq \lambda,
$$
is a duality gap for (\ref{eq:formulation_with_X}).
where $f^*(\kappab)\defin\sup_{\z} [\z^\top\kappab - f(\z)]$ is the Fenchel conjugate of $f$.
Given a primal variable $\w$, a good dual candidate $\kappab$ can be obtained by looking at the conditions
that have to be satisfied by the pair $(\w,\kappab)$ at optimality~\cite{borwein}.
In particular, the dual variable $\kappab$ is chosen to be  
$$
\kappab = -\rho^{-1} \nabla\! f(\X\w),\ \text{with}\ \rho\defin \max\big\{\lambda^{-1}\Omega^*(\X^\top\! \nabla\! f(\X\w)),1\big\}.
$$
Consequently, computing the duality gap requires evaluating the dual norm $\Omega^*$. 
We sum up the computation of the duality gap in~Algorithm~\ref{alg:fista}.

Moreover, we refer to the proximal operator associated with $\lambda \Omega$ as
$\text{prox}_{[\lambda\Omega]}$. As a brief reminder, it is defined as the
function that maps the vector~$\u$ in~$\R{p}$ to the (unique, by strong
convexity) solution of Eq.~(\ref{eq:prox_problem}).

\begin{algorithm}[hbtp]
\caption{FISTA procedure to solve problem~(\ref{eq:formulation_with_X}).}\label{alg:fista}
\begin{algorithmic}[1]
\STATE \textbf{Inputs}: initial $\w_{(0)} \in \R{p}$, $\Omega$, $\lambda>0$, $\varepsilon_{\text{gap}}>0$ (precision for the duality gap).
\STATE \textbf{Parameters}: $\nu > 1$, $L_{0} > 0$.
\STATE \textbf{Outputs}: solution $\w$.
\STATE \textbf{Initialization}: $\y_{(1)}=\w_{(0)}$, $t_1=1$, $k=1$.
\WHILE{ $\big\{$ \texttt{computeDualityGap}$\big(\w_{(k-1)}\big) > \varepsilon_{\text{gap}} \big\}$ }
\STATE Find the smallest integer $s_k\! \geq\! 0$ such that

\STATE $\quad f(\text{prox}_{[\lambda\Omega]}(\y_{(k)})) \leq f(\y_{(k)}) + \Delta_{(k)}^\top \nabla f(\y_{(k)}) + \frac{\tilde{L}}{2} \| \Delta_{(k)} \|_2^2,$
\STATE $\quad$with $\tilde{L}\defin L_{k}\nu^{s_k}$ and $\Delta_{(k)}\defin \y_{(k)}\!-\!\text{prox}_{[\lambda\Omega]}(\y_{(k)})$.
\STATE $L_{k}    \leftarrow  L_{k-1} \nu^{s_k}$.
\STATE $\w_{(k)}  \leftarrow \text{prox}_{[\lambda\Omega]}(\y_{(k)})$.
\STATE $t_{k+1}    \leftarrow (1+\sqrt{1+t_k^2})/2$.
\STATE $\y_{(k+1)} \leftarrow \w_{(k)} + \frac{t_k-1}{t_{k+1}} (\w_{(k)} - \w_{(k-1)})$.
\STATE $k \leftarrow k + 1$.
\ENDWHILE
\STATE \textbf{Return}: $\w \leftarrow  \w_{(k-1)}$.
\end{algorithmic}
\vspace*{0.2cm}
{\bf Procedure} \texttt{computeDualityGap}($\w$)
\begin{algorithmic}[1]
 \STATE $\kappab \leftarrow  -\rho^{-1} \nabla\! f(\X\w),\ \text{with}\ \rho\defin \max\big\{\lambda^{-1}\Omega^*(\X^\top\! \nabla\! f(\X\w)),1\big\}$.
 \STATE \textbf{Return}: $f(\X\w)+ \lambda \Omega(\w) + f^*(-\kappab)$.
\end{algorithmic}

\end{algorithm}

\section{Additional Experimental Results} \label{appendix:exp}
\subsection{Speed comparison of Algorithm \ref{algo:prox} with parametric max-flow algorithms}
\label{sec:speed_comp}
As shown in~\cite{hochbaum}, min-cost flow problems, and in particular, the dual problem of~(\ref{eq:prox_problem}), 
can be reduced to a specific \emph{parametric max-flow} problem.
We thus compare our approach (ProxFlow) with
the efficient parametric max-flow algorithm proposed by Gallo, Grigoriadis, and Tar-
jan ~\cite{gallo} and a simplified version of the latter proposed by Babenko and Goldberg in ~\cite{babenko}.
We refer to these two algorithms as GGT and SIMP respectively. 
The benchmark is established on the same datasets as those already used in the experimental section of the paper, 
namely: (1) three datasets built from overcomplete bases of discrete cosine transforms (DCT), with respectively 
 $10^4,\ 10^5$ and $10^6$ variables, and (2) images used for the background subtraction task, composed of 57600 pixels.
For GGT and SIMP, we use the \texttt{paraF} software which is a \texttt{C++} parametric max-flow implementation available at \texttt{http://www.avglab.com/andrew/soft.html}. Experiments were conducted on a single-core 2.33 Ghz.
 
We report in the following table the execution time in seconds of each algorithm, as well as the statistics of the corresponding problems:\\
\begin{center}
\begin{tabular}{|c||c|c|c|c|}
\hline
Number of variables $p$ & $10\,000$ & $100\,000$ & $1\,000\,000$ & $57\,600$ \\
\hline\hline
$|V|$ & $20\,000$ & $200\,000$ & $2\,000\,000$ & $75\,600$ \\
\hline
$|E|$ & $110\,000$ & $500\,000$ & $11\,000\,000$ & $579\,632$ \\ 
\hline
ProxFlow (in sec.)  & $\mathbf{0.4}$ & $\mathbf{3.1}$ & $\mathbf{113.0}$ & $\mathbf{1.7}$ \\
\hline
GGT (in sec.) & $2.4$ & $26.0$ & $525.0$ & $16.7$ \\
 \hline
SIMP (in sec.) & $1.2$ & $13.1$ & $284.0$ & $8.31$ \\ 
\hline
\end{tabular} 
\end{center}
Although we provide the speed comparison for a single value of $\lambda$ (the one used in the corresponding experiments of the paper), 
we observed that our approach consistently outperforms GGT and SIMP for values of $\lambda$ corresponding to different regularization regimes.

% 
 % small is acceptable
 \bibliographystyle{unsrt}
 \bibliography{RR-7372}

\end{document}